\documentclass[11pt]{article}
\usepackage{amsmath}
\usepackage{amsfonts}
\usepackage{amssymb}
\usepackage{color}
\usepackage{lscape}
\usepackage{latexsym}
\usepackage{multirow}
\usepackage{graphicx}
\newtheorem{theorem}{\bf Theorem}[section]
\newtheorem{lemma}[theorem]{\bf Lemma}
\newtheorem{prop}[theorem]{\bf Proposition}

\newtheorem{coro}[theorem]{\bf Corollary}

\newenvironment{proof}{\noindent{\em Proof:}}{\quad \hfill$\Box$\vspace{2ex}}

\def \bN {\Bbb N}

\def \bR {\Bbb R}

\def \bJ {\Bbb J}

\def \bC {\Bbb C}

\def \bz {{\bf z}}

\def \cB {{\cal B}}
\def \cC {{\cal C}}

\def \cF {{\cal F}}
\def \cH {{\cal H}}

\def \cL {{\cal L}}

\def \cS {{\cal S}}

\def \cP {{\cal P}}

\def \cE {{\cal E}}

\def \cZ {{\cal Z}}

\def \cW {{\cal W}}

\def \tK {\tilde{K}}
\def \tG {\tilde{G}}
\def \and {\, \mbox{\rm and}\, }

\def \span {\,{\rm span}\,}

\def \ran {{\rm ran}\,}
\def \Re {\,{\rm Re}\,}

\def \ae {\mbox{a.e. }}
\makeatletter

\newcommand{\Rmnum}[1]{\expandafter\@slowromancap\romannumeral #1@}
\makeatother
\begin{document}

\title{\bf
Refinement of Operator-valued Reproducing Kernels\thanks{Supported
by Guangdong Provincial Government of China through the
``Computational Science Innovative Research Team" program.}}
\author{Yuesheng Xu\thanks{Department of Mathematics, Syracuse University, Syracuse, NY 13244,
USA, and Department of Scientific Computing and Computer
Applications, Sun Yat-sen University, Guangzhou 510275, P. R. China.
E-mail address: {\it yxu06@syr.edu}. Supported in part by US Air
Force Office of Scientific Research under grant FA9550-09-1-0511, by
the US National Science Foundation under grants DMS-0712827, by the
Natural Science Foundation of China under grant 11071286.},\quad
Haizhang Zhang\thanks{Corresponding author. School of Mathematics
and Computational Science and Guangdong Province Key Laboratory of Computational Science,
 Sun Yat-sen University, Guangzhou 510275, P. R. China. E-mail address: {\it zhhaizh2@sysu.edu.cn}.},\quad and \quad Qinghui Zhang
\thanks{Department of Scientific Computing and Computer
Applications, School of Mathematics and Computational Science,
 Sun Yat-sen University, Guangzhou 510275, P. R. China. E-mail address: {\it zhqingh@mail2.sysu.edu.cn}.}}
\date{}
\maketitle

\begin{abstract}
This paper studies the construction of a {\it refinement} kernel for
a given operator-valued reproducing kernel such that the
vector-valued reproducing kernel Hilbert space of the refinement
kernel contains that of the given one as a subspace. The study is
motivated from the need of updating the current operator-valued
reproducing kernel in multi-task learning when underfitting or
overfitting occurs. Numerical simulations confirm that the
established refinement kernel method is able to meet this need.
Various characterizations are provided based on feature maps and
vector-valued integral representations of operator-valued
reproducing kernels. Concrete examples of refining translation
invariant and finite Hilbert-Schmidt operator-valued reproducing
kernels are provided. Other examples include refinement of Hessian
of scalar-valued translation-invariant kernels and transformation
kernels. Existence and properties of operator-valued reproducing
kernels preserved during the refinement process are also
investigated.
\end{abstract}

\noindent{\bf Keywords}: vector-valued reproducing kernel Hilbert
spaces, vector-valued reproducing kernels, refinement, embedding,
translation invariant kernels, Hessian of Gaussian kernels,
Hilbert-Schmidt kernels, numerical experiments.

\vspace*{0.3cm}

\section{Introduction}
\setcounter{equation}{0}

Machine learning designs algorithms for the purpose of inferring
from finite empirical data a function dependency which can then be
used to understand or predict generation of new data. Past research
has mainly focused on single task learning problems where the
function to be learned is scalar-valued. Built upon the theory of
scalar-valued reproducing kernels \cite{Aronszajn}, kernel methods
have proven useful in single task learning, \cite{ScSm,ShCr,Vapnik}.
The approach might be justified in three ways. Firstly, as inputs
for learning algorithms are sample data, requiring the sampling
process to be stable seems inevitable. Thanks to the existence of an
inner product, Hilbert spaces are the class of normed vector spaces
that we can handle best. These two considerations lead
immediately to the notion of reproducing kernel Hilbert spaces
(RKHS). Secondly, a reasonable learning scheme is expected to make
use of the similarity between a new input and the existing inputs
for prediction. Inner products provide a natural measurement of
similarities. It is well-known that a bivariate function is a
scalar-valued reproducing kernel if and only if it is representable
as some inner product of the feature of inputs \cite{ScSm}. Finally,
finding a feature map and taking the inner product of the feature of
two inputs are equivalent to choosing a scalar-valued reproducing
kernel and performing function evaluations of it. This brings
computational efficiency and gives birth to the important ``kernel
trick" \cite{ScSm} in machine learning. For references on single
task learning and scalar-valued RKHS, we recommend
\cite{Aronszajn,CuckerSmale,CZ,EPP,ScSm,ShCr,Vapnik,ZXZ}.

In this paper, we are concerned with multi-task learning where the
function to be reconstructed from finite sample data takes range in
a finite-dimensional Euclidean space, or more generally, a Hilbert
space. Motivated by the success of kernel methods in single task
learning, it was proposed in \cite{EMP,MP2005} to develop algorithms
for multi-task learning in the framework of vector-valued RKHS. We
attempt to contribute to the theory of vector-valued RKHS by
studying a special embedding relationship between two vector-valued
RKHS. We shall briefly review existing work on vector-valued RKHS
and the associated operator-valued reproducing kernels. The study of
vector-valued RKHS dates back to \cite{Pedrick}. The notion of
matrix-valued or operator-valued reproducing kernels was also
obtained in \cite{BM}. References \cite{MW,MZ,YC} were devoted to
learning a multi-variate function and its gradient simultaneously.
Reference \cite{CDT} established the Mercer theorem for
vector-valued RKHS and characterized those spaces with elements
being $p$-integrable vector-valued functions. Various
characterizations and examples of universal operator-valued
reproducing kernels were provided in \cite{CMPY,CDTU}. The latter
\cite{CDTU} also examined basic operations of operator-valued
reproducing kernels and extended the Bochner characterization of
translation invariant reproducing kernels to the operator-valued
case.

The purpose of this paper is to study the refinement relationship of
two vector-valued reproducing kernels. We say that a vector-valued
reproducing kernel is a refinement of another kernel of such type if
the RKHS of the first kernel contains that of the latter one as a
linear subspace and their norms coincide on the smaller space. The
precise definition will be given in the next section after we
provide necessary preliminaries on vector-valued RKHS. The study is
motivated by the need of updating a vector-valued reproducing kernel
for multi-task machine learning when underfitting or overfitting
occurs. Detailed explanations of this motivation will be presented
in the next section. Mathematically, a thorough understanding of the
refinement relationship is essential to the establishment of a
multi-scale decomposition of vector-valued RKHS, which in turn is
the foundation for extending multi-scale analysis
\cite{Daubechies,Mallat} to kernel methods. In fact, a special
refinement method by a bijective mapping from the input space to
itself provides such a decomposition. As the procedure is similar to
the scalar-valued case, we refer interested authors to \cite{XZ} for
the details. The notion of refinement of scalar-valued kernels was
initiated and extensively investigated by the first two authors
\cite{XZ,XZ2}. Therefore, a general principle we shall follow is to
briefly mention or even completely omit arguments that are not
essentially different from the scalar-valued case. As we proceed
with the study, it will become clear that nontrivial obstacles in
extending the scalar-valued theory to vector-valued RKHS are mainly
caused by the complexity in the vector-valued integral
representation of the operator-valued reproducing kernels under
investigation, by the complicated form of the feature map involved,
which is also operator-valued, and by the infinite-dimensionality of
the output space in some occasions.

This paper is organized as follows. We shall introduce necessary
preliminaries on vector-valued RKHS and motivate our study from
multi-tasking learning in the next section. In Section 3, we shall
present three general characterizations of the refinement
relationship by examining the difference of two given kernels, the
feature map representation of kernels, and the associated kernels on
the extended input space. Recall that most scalar-valued reproducing
kernels are represented by integrals. In the operator-valued case,
we have two types of integral representations: the integral of
operator-valued reproducing kernels with respect to a scalar-valued
measure, and the integral of scalar-valued reproducing kernels with
respect to an operator-valued measure. As a key part of this paper,
we shall investigate in Section 4 specifications of the general
characterizations when the operator-valued reproducing kernels are
given by such integrals. In Section 5, we present concrete examples
of refinement by looking into translation-invariant operator-valued
kernels, Hessian of a scalar-valued kernels, Hilbert-Schmidt
kernels, etc. Section 6 treats specially the existence of nontrivial
refinements and desirable properties of operator-valued reproducing
kernels that can be preserved during the refinement process. In
Section 7, we perform two numerical simulations to show the effect
of the refinement kernel method in updating operator-valued
reproducing kernels for multi-task learning. Finally, we conclude
the paper in Section 8.

\section{Kernel Refinement}
\setcounter{equation}{0}

To explain our motivation from multi-task learning in details, we first recall the definition of operator-valued reproducing kernels. Throughout the paper, we let $X$ and $\Lambda$ denote a prescribed set and a separable Hilbert space, respectively. We shall call $X$ the input space and $\Lambda$ the output space. To avoid confusion, elements in $X$ and $\Lambda$ will be denoted by $x,y$, and $\xi,\eta$, respectively. Unless specifically mentioned, all the normed vector spaces in the paper are over the field $\bC$ of complex numbers. Let $\cL(\Lambda)$ be the set of all the bounded linear operators from $\Lambda$ to $\Lambda$, and $\cL_+(\Lambda)$ its subset of those linear operators $A$ that are positive, namely,
$$
(A\xi,\xi)_\Lambda\ge0\mbox{ for all }\xi\in \Lambda,
$$
where $(\cdot,\cdot)_\Lambda$ is the inner product on $\Lambda$. The adjoint of $A\in \cL(\Lambda)$ is denoted by $A^*$. An {\it $\cL(\Lambda)$-valued reproducing kernel} on $X$ is a function $K:X\times X\to \cL(\Lambda)$ such that $K(x,y)=K(y,x)^*$ for all $x,y\in X$, and such that for all $x_j\in X$, $\xi_j\in \Lambda$, $j\in\bN_n:=\{1,2,\ldots,n\}$, $n\in\bN$,
\begin{equation}\label{positivity}
\sum_{j=1}^n\sum_{k=1}^n(K(x_j,x_k)\xi_j,\xi_k)_\Lambda\ge0.
\end{equation}
For each $\cL(\Lambda)$-valued reproducing kernel $K$ on $X$, there exists a unique Hilbert space, denoted by $\cH_K$, consisting of $\Lambda$-valued functions on $X$ such that
\begin{equation}\label{inclusion}
K(x,\cdot)\xi\in\cH_K\mbox{ for all }x\in X\mbox{ and }\xi\in\Lambda
\end{equation}
and
\begin{equation}\label{reproducing}
(f(x),\xi)_\Lambda=(f,K(x,\cdot)\xi)_{\cH_K}\mbox{ for all }f\in\cH_K,\ x\in X,\mbox{ and }\xi\in\Lambda.
\end{equation}
It is implied by the above two properties that the point evaluation at each $x\in X$:
$$
\delta_x(f):=f(x),\ \ f\in\cH_K
$$
is continuous from $\cH_K$ to $\Lambda$. In other words, $\cH_K$ is a $\Lambda$-valued RKHS. We call it the RKHS of $K$. Conversely, for each $\Lambda$-valued RKHS on $X$, there exists a unique $\cL(\Lambda)$-valued reproducing kernel $K$ on $X$ that satisfies (\ref{inclusion}) and (\ref{reproducing}). For this reason, we also call $K$ the reproducing kernel (or kernel for short) of $\cH_K$. The bijective correspondence between $\cL(\Lambda)$-valued reproducing kernels and $\Lambda$-valued RKHS is central to the theory of vector-valued RKHS.

Given two $\cL(\Lambda)$-valued reproducing kernels $K,G$ on $X$, we shall investigate in this paper the fundamental embedding relationship $\cH_K\preceq\cH_G$ in the sense that $\cH_K\subseteq\cH_G$ and for all $f\in\cH_K$, $\|f\|_{\cH_K}=\|f\|_{\cH_G}$. Here, $\|\cdot\|_\cW$ denotes the norm of a normed vector space $\cW$. We call $G$ a {\it refinement} of $K$ if there does hold $\cH_K\preceq \cH_G$. Such a refinement is said to be nontrivial if $G\ne K$.

We motivate this study from the kernel methods for multi-task learning and from the multi-scale decomposition of vector-valued RKHS. Let $\bz:=\{(x_j,\xi_j):j\in\bN_n\}\subseteq X\times \Lambda$ be given sample data. A typical kernel method infers from $\bz$ the minimizer $f_\bz$ of
\begin{equation}\label{regularization}
\min_{f\in\cH_K}\frac 1n \sum_{j=1}^n C(x_j,\xi_j,f(x_j))+\sigma \phi(\|f\|_{\cH_K}),
\end{equation}
where $K$ is a selected $\cL(\Lambda)$-valued reproducing kernel on $X$, $C$ a prescribed loss function, $\sigma$ a positive regularization parameter, and $\phi$ a regularizer. The ideal predictor $f_0:X\to \Lambda$ that we are pursuing is the one that minimizes
$$
\cE(f):=\int_{X\times \Lambda}C(x,\xi,f(\xi))dP
$$
among all possible functions $f$ from $X$ to $\Lambda$. Here $P$ is an unknown probability measure on $X\times\Lambda$ that dominates the generation of data from $X\times \Lambda$. We wish that $\cE(f_z)-\cE(f_0)$ can converge to zero in probability as the number $n$ of sampling points tends to infinity. Whether this will happen depends heavily on the choice of the kernel $K$. The error $\cE(f_z)-\cE(f_0)$ can be decomposed into the sum of the {\it approximation error} and {\it sampling error}, \cite{ScSm,Vapnik}. The approximation error occurs as we search the minimizer in a restricted set of candidate functions, namely, $\cH_K$. It becomes smaller as $\cH_K$ enlarges. The sampling error is caused by replacing the expectation $\cE(f)$ of the loss function $C(x,\xi,f(\xi))$ with the sample mean
$$
\frac 1n \sum_{j=1}^nC(x_j,\xi_j,f(x_j)).
$$
By the law of large numbers, the sample mean converges to the expectation in probability as $n\to\infty$ for a fixed $f\in\cH_K$. However, as $f_\bz$ varies according to changes in the sample data $\bz$, we need a uniform version of the law of large number on $\cH_K$ in order to well control the sampling error. Therefore, the sampling error usually increases as $\cH_K$ enlarges, or to be more precisely, as the {\it capacity} of $\cH_K$ increases.

By the above analysis, we might encounter two situations after the choice of an $\cL(\Lambda)$-valued reproducing kernel $K$:
\begin{itemize}
\item[---] overfitting, which occurs when the capacity of $\cH_K$ is too large, forcing the minimizer obtained from (\ref{regularization}) to imitate artificial function dependency in the sample data, and thus causing the sampling error to be out of control;
\item[---] underfitting, which occurs when $\cH_K$ is too small for the minimizer of (\ref{regularization}) to describe the desired function dependency implied in the data, and thus failing in bounding the approximation error.
\end{itemize}
When one of the above situations happens, a remedy is to modify the reproducing kernel. Specifically, one might want to find another $\cL(\Lambda)$-valued reproducing kernel $G$ such that $\cH_K\preceq\cH_G$ when there is underfitting, or such that $\cH_G\preceq \cH_K$ when there is overfitting. We see that in either case, we need to make use of the refinement relationship. We shall verify in the last section through extensive numerical simulations that the refinement kernel method is indeed able to provide an appropriate update of an operator-valued reproducing kernel when underfitting or overfitting occurs.

\section{General Characterizations}
\setcounter{equation}{0}

The relationship between the RKHS of the sum of two operator-valued reproducing kernels and those of the summand kernels has been made clear in Theorem 1 on page 44 of \cite{Pedrick}. Our first characterization of refinement is a direct consequence of this result.

\begin{prop}\label{firstcharacterization}
Let $K,G$ be two $\cL(\Lambda)$-valued reproducing kernels on $X$.
Then $\cH_K\preceq\cH_G$ if and only if $G-K$ is an $\cL(\Lambda)$-valued
reproducing kernel on $X$ and $\cH_K \cap \cH_{G-K} = \{0\}$. If
$\cH_K\preceq\cH_G$ then $\cH_{G-K}$ is the orthogonal complement of ${\cal
H}_K$ in ${\cal H}_G$.
\end{prop}

Every reproducing kernel has a feature map representation. Specifically, $K$ is an $\cL(\Lambda)$-valued reproducing kernel on $X$ if and only if there exists a Hilbert space $\cW$ and a mapping $\Phi:X\to \cL(\Lambda,\cW)$ such that
\begin{equation}\label{featureK}
K(x,y)=\Phi(y)^*\Phi(x),\ \ x,y\in X,
\end{equation}
where $\cL(\Lambda,\cW)$ denotes the set of bounded linear operators from $\Lambda$ to $\cW$, and $\Phi(y)^*$ is the adjoint operator of $\Phi(y)$. We call $\Phi$ a {\it feature map} of $K$. The following lemma is useful in identifying the RKHS of a reproducing kernel given by a feature map representation (\ref{featureK}).

\begin{lemma}\label{rkhsfeature}
If $K$ is an $\cL(\Lambda)$-valued reproducing kernel on $X$ given by (\ref{featureK}) then
\begin{equation}\label{rkhsfeatureeq1}
\cH_K = \{\Phi(\cdot)^*u:\,u\in\cW\}
\end{equation}
with inner product
$$
(\Phi(\cdot)^*u,\Phi(\cdot)^*v)_{\cH_K}:= (P_\Phi u,P_\Phi v)_\cW,\ \ u,v\in\cW,
$$
where $P_\Phi$ is the orthogonal projection of $\cW$ onto
$$
\cW_\Phi:=\overline{\mbox{span}}\{\Phi(x)\xi:\,x\in
X,\,\xi\in\Lambda\}.
$$
\end{lemma}

The second characterization can be proved using Lemma \ref{rkhsfeature} and the same arguments with those for the scalar-valued case \cite{XZ}.

\begin{theorem}\label{secondcharacterization}
Suppose that $\cL(\Lambda)$-valued reproducing kernels $K$ and $G$
are given by the feature maps
 $\Phi:X\to \cL(\Lambda,\cW)$ and $\Phi':X\to
\cL(\Lambda,\cW')$, respectively. Assume that $\cW_\Phi = \cW$ and
$\cW'_{\Phi'} = \cW'$. Then $\cH_K\preceq\cH_G$ if and only if there exists
a bounded linear operator $T$ from $\cW'$ to $\cW$ such that
\begin{equation}\label{Iso}
T\Phi'(x) =\Phi(x)\mbox{ for all }x\in X,
\end{equation}
and the adjoint operator $T^*:\cW\to\cW'$ is isometric. In this case,
$G$ is a nontrivial refinement of $K$ if and only if $T$ is not
injective.
\end{theorem}

To illustrate the above useful results, we shall present a concrete example aiming at refining $\cL(\Lambda)$-valued reproducing kernels $K$
with a finite-dimensional RKHS. A simple observation is
made regarding such a kernel.

\begin{prop}
A $\Lambda$-valued RKHS $\cH_K$ is of finite dimension $n\in \bN$ if and
only if there exists an $n\times n$ hermitian and strictly
positive-definite matrix $A$ and $n$ linearly independent functions
$\phi_j:X\to\Lambda$, $j\in \bN_n$ such that
\begin{equation}\label{FiniKer}
K(x,y)\xi = \sum_{j=1}^n\sum_{k=1}^n
A_{jk}(\xi,\phi_j(x))_\Lambda\phi_k(y),\ \ x,y\in X,\ \xi\in
\Lambda.
\end{equation}
\end{prop}
\begin{proof}
Assume that $\cH_K$ is $n$ dimensional with orthogonal basis
$\{\phi_j:j\in \bN_n\}$. As $K(x,\cdot)\xi\in \cH_K$ for all
$x\in X$, $\xi\in\Lambda$, there exist functions
$c_j:X\times\Lambda\to\mathbb C$ such that
$$
K(x,y)\xi =
\sum_{j=1}^nc_j(\xi,x)\phi_j(y),\ \ x,\,y\in X,\ \xi\in \Lambda.
$$
Since $\{\phi_j:j\in \bN_n\}$ is a basis for $\cH_K$, each
function $f\in \cH_K$ has the form
$$
f=\sum_{j=1}^nd_j\phi_j,\,\,d_j\in{\mathbb C},\,\,j\in \bN_n.
$$
Clearly, $\|f\|:=\big(\sum_{j=1}^n|d_j|^2\big)^{1/2}$ is a norm on
$\cH_K$. It is equivalent to the original one on $\cH_K$ as
$\dim\cH_K<\infty$. It is implied that there exists some $C>0$
such that
\begin{equation}\label{est1}
\sum_{j=1}^n|c_j(\xi,x)|^2\le C\|K(x,\cdot)\xi\|_{\cH_K}^2 =
C(K(x,x)\xi,\xi)_\Lambda\le C\|\xi\|_\Lambda^2\|K(x,x)\|.
\end{equation}
Obviously, for each $x\in X$ and $j\in\bN_n$, $c_j(\cdot,x)$
is a linear functional on $\Lambda$. This together with (\ref{est1})
implies that $c_j(\cdot,x)$ are bounded linear functionals on
$\Lambda$. By the Riesz representation theorem, there exists
$\psi_j: X\to\Lambda,\,j\in \bN_n$ such that
$$
c_j(\xi,x) = (\xi,\psi_j(x))_\Lambda.
$$
We conclude that $K$ has the form
\begin{equation}\label{est2}
K(x,y)\xi = \sum_{j=1}^n(\xi,\psi_j(x))_\Lambda\phi_j(y),\ \ x,\,y\in X,\ \xi\in \Lambda.
\end{equation}
Since $\{\phi_j:j\in \bN_n\}$ is an orthogonal basis for $\cH_K$, by
(\ref{reproducing}),
$$
(\xi,\psi_j(x))_\Lambda = (K(x,\cdot)\xi,\phi_j)_{\cH_K} =
(\xi,\phi_j(x))_\Lambda,\,\,\xi\in\Lambda,\,x\in X.
$$
It follows that $\psi_j = \phi_j,\,\,j\in \bN_n$.
Substituting this into (\ref{est2}) yields that
$$
K(x,y)\xi = \sum_{j=1}^n(\xi,\phi_j(x))_\Lambda\phi_j(y),\ \ x,\,y\in X,\ \xi\in \Lambda,
$$
which indeed is a special form of (\ref{FiniKer}).

Conversely, assume that $K$ has the form (\ref{FiniKer}). We set
$\cW_A:=I_A^2(\bN_n):=\{c=(c_j:j\in\bN_n)\in{\mathbb
C}^n\}$ with inner product
$$
(c,d)_{I_A^2(\bN_n)}:= \sum_{j=1}^n\sum_{k=1}^nc_j\bar{d}_kA_{jk}.
$$
Introduce $\Phi:X\to \cL(\Lambda,\cW_A)$ by setting
$\Phi(x)\xi:=\big((\xi,\phi_j(x))_\Lambda:j\in\bN_n\big)$.
Direct computations show that
$$
\Phi^*(x)u = \sum_{j=1}^n\sum_{k=1}^n\phi_j(x)u_kA_{jk},\,\,u=(u_j:j\in\bN_n)\in\cW_A.
$$
Thus, we see that $K(x,y) = \Phi(y)^*\Phi(x)$, $x,\,y\in X$,
implying that $K$ is an $\cL(\Lambda)$-valued reproducing kernel. By
the linear independence of $\phi_j,\,j\in \bN_n$,
$\mbox{span}\{\Phi(x)\xi:x\in X,\,\xi\in \Lambda\} = \cW_A$. We
hence apply Lemma \ref{rkhsfeature} to get that
$$
\cH_K = \{\Phi(\cdot)^*u:u\in\cW_A\}=\mbox{span}\{\phi_j:j\in\bN_n\},
$$
which is of dimension $n$.
\end{proof}

By the above proposition, we let $\phi_j$, $j\in\bN_m$ be linearly
independent functions from $X$ to $\Lambda$, where $m\ge n$ are
fixed positive integers. Let $A$ and $B$ be $n\times n$ and $m\times
m$ hermitian and strictly positive-definite matrices, respectively.
We define $K$ by (\ref{FiniKer}) in terms of matrix $A$ and $G$ by
\begin{equation}\label{FormG}
G(x,y)\xi:=\sum_{j=1}^m\sum_{k=1}^mB_{jk}(\xi,\phi_j(x))_\Lambda\phi_k(y),\,\,x,\,y\in
X
\end{equation}
and shall investigate conditions for $G$ to be a refinement of $K$.

\begin{prop}
Let $K$, $G$ be defined by (\ref{FiniKer}) and (\ref{FormG}),
respectively. Then $\cH_K\preceq \cH_G$ if and only if $B^{-1}$ is
an augmentation of $A^{-1}$, namely, $B^{-1}_{jk} = A^{-1}_{jk}$,
$j,\,k\in \bN_n$. In particular, if $K$, $G$ have the form
$$
K(x,y)\xi = \sum_{j\in \bN_n}a_j(\xi,\phi_j(x))_\Lambda\phi_j(y),\,\,
G(x,y)\xi = \sum_{k\in \bN_m}b_k(\xi,\phi_k(x))_\Lambda\phi_k(y)
$$
for some positive constants $a_j$, $b_k$, then
$\cH_K\preceq \cH_G$ if and only if $a_j = b_j,\,j\in\bN_n$.
In both cases if $\cH_K\preceq \cH_G$ then $G$ is a nontrivial
refinement of $K$ if and only if $m>n$.
\end{prop}

\begin{proof}
It suffices to prove the first claim. We observe that $K$, $G$ have
the feature spaces $\cW = I_A^2(\bN_n)$ and $\cW' =
I_B^2(\bN_m)$, respectively, with feature maps
\[\Phi(x)\xi:=\big((\xi,\phi_j(x))_\Lambda:j\in\bN_n\big),\,\,
\Phi'(x)\xi:=\big((\xi,\phi_k(x))_\Lambda:k\in\bN_m\big),\,\,x\in X,\,\xi\in\Lambda.\] Suppose that
$\cH_K\preceq\cH_G$, then by Theorem \ref{secondcharacterization}, there exists a
bounded linear operator $T:\cW'\to\cW$ with properties as described
there. It can be represented by an $n\times m$ matrix $D$ as
\begin{equation}
(T\Phi'(x)\xi)_j = \sum_{k=1}^mD_{jk}(\xi,\phi_k(x))_\Lambda =
(\xi,\phi_j(x))_\Lambda,\ \ x\in X,\xi\in \Lambda,
\end{equation}
which implies that $D = [I_n,0]$, where $I_n$ denotes the $n\times
n$ identity matrix. The adjoint operator $T^*$ of $T$ is then
represented by
\[T^*u = B^{-1}\left[\begin{array}{l}A\\0\end{array}\right]u,\,\,u\in {\mathbb C}^n.\]
Since $T^*$ is isometric, we get that
\[(T^*u,T^*v)_{\cW'} = (u,v)_\cW,\]
which has the form
\[v^*[A,0]B^{-1}BB^{-1}\left[\begin{array}{l}A\\0\end{array}\right]u = v^*Au,\,\,u,\,v\in {\mathbb C}^n.\]
We derive from the above equation that
\[[A,0]B^{-1}\left[\begin{array}{l}A\\0\end{array}\right] = A.\]
Therefore, $B^{-1}$ is an augmentation of $A^{-1}$. Conversely, if
this is true then $T:\cW'\to\cW$ defined by
$$
Tu': = [I_n, 0]u',\,u'\in {\mathbb C}^m
$$
satisfies the two properties in Theorem \ref{secondcharacterization}. As a result,
$\cH_K\preceq \cH_G$.
\end{proof}

It is worthwhile to point out that the above characterization is independent of the Hilbert space $\Lambda$.

Unlike the previous two characterizations, the third one comes as a surprise, telling us that theoretically we are able to reduce our consideration to the scalar-valued case.

Introduce for each $\cL(\Lambda)$-valued reproducing kernel $K$ on $X$ a scalar-valued reproducing kernel $\tilde{K}$ on the {\it extended input space} $\tilde{X}:=X\times\Lambda$ by setting
$$
\tilde{K}((x,\xi),(y,\eta)):=(K(x,y)\xi,\eta)_\Lambda,\ \ x,y\in X,\ \xi,\eta\in\Lambda.
$$
By (\ref{positivity}), $\tilde{K}$ is indeed positive-definite.

\begin{prop}\label{thirdcharacterization}
There holds $\cH_K\preceq\cH_G$ if and only if $\cH_{\tilde{K}}\preceq\cH_{\tilde{G}}$. Furthermore, $G$ is a nontrivial refinement of $K$ on $X$ if and only if $\tG$ is a nontrivial refinement of $\tK$ on $\tilde{X}$.
\end{prop}
\begin{proof}
We first explore the close relationship between $\cH_K$ and $\cH_{\tilde{K}}$. By (\ref{reproducing}),
$$
\tilde{K}((x,\xi),(y,\eta))=(K(x,y)\xi,\eta)_\Lambda=(K(x,\cdot)\xi,K(y,\cdot)\eta)_{\cH_K},
$$
which provides a natural feature map $\Phi:\tilde{X}\to\cH_K$ of $\tilde{K}$
$$
\Phi((x,\xi)):=K(x,\cdot)\xi,\ \ x\in X,\ \xi\in\Lambda.
$$
The density condition $\cW_\Phi=\cH_K$ is clearly satisfied by (\ref{reproducing}). We hence obtain by (\ref{rkhsfeature}) that every function $\tilde{f}$ in $\cH_K$ is of the form
$$
\tilde{f}(x,\xi):=(f(x),\xi)_\Lambda\mbox{ for some }f\in\cH_K
$$
with
$$
\|\tilde{f}\|_{\cH_{\tilde{K}}}=\|f\|_{\cH_K}.
$$
Similar observations can be made about $\cH_{\tilde{G}}$.

It follows immediately that $\cH_{\tilde{K}}\preceq\cH_{\tilde{G}}$ if $\cH_K\preceq\cH_G$. On the other hand, suppose that $\cH_{\tilde{K}}\preceq\cH_{\tilde{G}}$. Then for each $f\in\cH_K$ there exists some $g\in\cH_G$ such that
\begin{equation}\label{thirdcharacterizationeq1}
(f(x),\xi)_\Lambda=\tilde{f}(x,\xi)=\tilde{g}(x,\xi)=(g(x),\xi)_\Lambda\mbox{ for all }x\in X,\ \xi\in\Lambda
\end{equation}
and
$$
\|f\|_{\cH_K}=\|\tilde{f}\|_{\cH_{\tilde{K}}}=\|\tilde{g}\|_{\cH_{\tG}}=\|g\|_{\cH_G}.
$$
Equation (\ref{thirdcharacterizationeq1}) implies that $f=g$. Therefore, $\cH_K\preceq\cH_G$.
\end{proof}

It appears by Proposition \ref{thirdcharacterization} that we do not have to bother studying refinement of operator-valued reproducing kernels. Although the strategy sometimes does simplify the problem, the difficulty is generally not reduced significantly. Instead, the result might be viewed as transferring the complexity to the input space. Moreover, desirable properties such as translation invariance of the original kernels might be lost in the process. As a result, an independent study of the operator-valued case remains necessary and challenging.

\section{Integral Representations}
\setcounter{equation}{0}

This section will be built on the theory of vector-valued measures and integrals \cite{Berberian,DU}. Necessary preliminaries on the subjects will be explained in sufficient details.

\subsection{Operator-valued kernels with respect to scalar-valued measures.} Let us first introduce integration of a vector-valued function with respect to a scalar-valued measure. Let $\cF$ be a $\sigma$-algebra of subsets of a fixed set $\Omega$, $\mu$ a finite nonnegative measure on $\cF$, and $\cB$ a Banach space. We are concerned with $\cB$-valued functions on $\Omega$. A function $f:\Omega\to\cB$ is said to be {\it simple} if
\begin{equation}\label{simplefunction}
f=\sum_{j=1}^n a_j \chi_{E_j}
\end{equation}
for some finitely many $a_j\in \cB$ and pairwise disjoint subsets $E_j\in \cF$, $j\in\bN_n$. A function $f:\Omega\to \cB$ is called $\mu$-{\it measurable} if there exists a sequence of $\cB$-valued simple functions $f_n$ on $\Omega$ such that
$$
\lim_{n\to\infty}\|f_n(t)-f(t)\|_{\cB}=0\mbox{ for }\mu-\ae t\in\Omega,
$$
where $\mu-\ae$stands for ``everywhere except for a set of zero $\mu$ measure". Finally, a $\cB$-valued function $f$ on $\Omega$ is called {\it $\mu$-Bochner integrable} if there exists a sequence of simple functions $f_n:\Omega\to\cB$ such that
\begin{equation}\label{bochnerintergable}
\lim_{n\to\infty}\int_\Omega\|f_n(t)-f(t)\|_{\cB}\,d\mu(t)=0.
\end{equation}
The integral of a simple function $f$ of the form (\ref{simplefunction}) on any $E\in\cF$ with respect to $\mu$ is defined by
$$
\int_E fd\mu:=\sum_{j=1}^n a_j\,\mu(E_j\cap E).
$$
In general, suppose that $f$ is a $\mu$-Bochner integrable function from $\Omega$ to $\cB$, that is, (\ref{bochnerintergable}) holds true. Then it is obvious that for each $E\in\cF$, $\int_E f_nd\mu$, $n\in\bN$ form a Cauchy sequence in $\cB$. Therefore,
$$
\int_E fd\mu:=\lim_{n\to\infty}\int_E f_nd\mu.
$$
The resulting integral $\int_E fd\mu$ is an element in $\cB$.

It is known that a $\mu$-measurable function $f:\Omega\to\cB$ is Bochner integrable if and only if
$$
\int_\Omega \|f(t)\|_\cB d\mu(t)<+\infty.
$$
This provides a way for us to comprehend the integral $\int_E fd\mu$ in the most needed case when $f$ is $\cL(\Lambda)$-valued. If $\cB=\cL(\Lambda)$ then we have for each $E\in\cF$ that
\begin{equation}\label{sesquilinear}
\left(\int_E fd\mu\, \xi,\eta\right)_\Lambda=\int_E (f(t)\xi,\eta)_\Lambda d\mu(t),\ \ \xi,\eta\in\Lambda.
\end{equation}
Clearly, the right hand side above defines a sesquilinear form on $\Lambda\times\Lambda$ which is bounded as
$$
\left|\int_E (f(t)\xi,\eta)_\Lambda d\mu(t)\right|\le \int_E \|f(t)\|_{\cL(\Lambda)}d\mu(t)\ \|\xi\|_\Lambda\|\eta\|_\Lambda,
$$
where $\|\cdot\|_{\cL(\Lambda)}$ is the operator norm on $\cL(\Lambda)$. As a result, (\ref{sesquilinear}) gives an equivalent way of defining the integral $\int_E fd\mu$ as a bounded linear operator on $\Lambda$, \cite{Conway}.

We introduce another notation before returning to reproducing kernels. Denote by $L^2(\Omega,\cB,d\mu)$ the Banach space of all the $\mu$-measurable functions $f:\Omega\to\cB$ such that
$$
\|f\|_{L^2(\Omega,\cB,d\mu)}:=\left(\int_\Omega \|f(t)\|_\cB^2d\mu(t)\right)^{1/2}<+\infty.
$$
When $\cB=\bC$, $L^2(\Omega,\bC,d\mu)$ will be abbreviated as $L^2(\Omega,d\mu)$. When $\cB$ is a Hilbert space, $L^2(\Omega,\cB,d\mu)$ is also a Hilbert space with the inner product
$$
(f,g)_{L^2(\Omega,\cB,d\mu)}:=\int_\Omega (f(t),g(t))_\cB d\mu(t),\ \ f,g\in L^2(\Omega,\cB,d\mu).
$$
The discussion in this section by far can be found in \cite{DU}.

Let $\mu,\nu$ be two finite nonnegative measures on a $\sigma$-algebra $\cF$ of subsets of $\Omega$. To introduce our $\cL(\Lambda)$-valued reproducing kernels, we also let $\cW$ be a Hilbert space and $\phi$ a mapping from $X\times \Omega$ to $\cL(\Lambda,\cW)$ such that for each $x\in X$, $\phi(x,\cdot)$ belongs to both $L^2(\Omega,\cL(\Lambda,\cW),d\mu)$ and $L^2(\Omega,\cL(\Lambda,\cW),d\nu)$. We shall investigate conditions that ensure $\cH_K\preceq\cH_G$ where
\begin{equation}\label{kernelkintegral1}
K(x,y)=\int_\Omega \phi(y,t)^*\phi(x,t)d\mu(t),\ \ x,y\in X
\end{equation}
and
\begin{equation}\label{kernelgintegral1}
G(x,y)=\int_\Omega \phi(y,t)^*\phi(x,t)d\nu(t),\ \ x,y\in X,
\end{equation}
where $\phi(y,t)^*$ is the adjoint operator of $\phi(y,t)$. Note that $K,G$ are well-defined as the integrand is Bochner integrable with respect to both $\mu$ and $\nu$. For instance, we observe by the Cauchy-Schwartz inequality for all $x,y\in X$ that
$$
\begin{array}{rl}
\displaystyle{\int_\Omega \|\phi(y,t)^*\phi(x,t)\|_{\cL(\Lambda)}d\mu(t)}&\displaystyle{\le \int_\Omega \|\phi(y,t)^*\|_{\cL(\cW,\Lambda)}\|\phi(x,t)\|_{\cL(\Lambda,\cW)}d\mu(t)}\\
&\displaystyle{=\int_\Omega \|\phi(y,t)\|_{\cL(\Lambda,\cW)}\|\phi(x,t)\|_{\cL(\Lambda,\cW)}d\mu(t)}\\
&\displaystyle{\le \|\phi(y,\cdot)\|_{L^2(\Omega,\cL(\Lambda,\cW),d\mu)}\|\phi(x,\cdot)\|_{L^2(\Omega,\cL(\Lambda,\cW),d\mu)}}.
\end{array}
$$
An alternative of expressing $K,G$ is for all $x,y\in X$, $\xi,\eta\in\Lambda$ that
$$
\tK((x,\xi),(y,\eta))=(K(x,y)\xi,\eta)_\Lambda=\int_\Omega (\phi(x,t)\xi,\phi(y,t)\eta)_\cW d\mu(t)
$$
and
$$
\tG((x,\xi),(y,\eta))=(G(x,y)\xi,\eta)_\Lambda=\int_\Omega (\phi(x,t)\xi,\phi(y,t)\eta)_\cW d\nu(t).
$$

When $\Lambda=\cW=\bC$, a characterization of $\cH_K\preceq\cH_G$ in terms of $\mu,\nu$ has been established in \cite{XZ2}. The relation, between the two measures, which we shall need is absolute continuity. We say that $\mu$ is {\it absolutely continuous} with respect to $\nu$ if for all $E\in\cF$, $\nu(E)=0$ implies $\mu(E)=0$. In this case, by the Radon-Nikodym theorem (see, \cite{Rudin}, page 121) for scalar-valued measures, there exists a nonnegative $\nu$-integrable function, denoted by $d\mu/d\nu$, such that
$$
\mu(E)=\int_E \frac{d\mu}{d\nu}(t)d\nu(t)\mbox{ for all }E\in\cF.
$$
We write $\mu\preceq\nu$ if $\mu$ is absolutely continuous with respect to $\nu$ and $d\mu/d\nu\in \{0,1\}$ $\nu-\ae$

When $\Lambda=\cW=\bC$, it was proved in Theorem 8 of \cite{XZ2} that if $\span\{\phi(x,\cdot):x\in X\}$ is dense in both $L^2(\Omega,d\mu)$ and $L^2(\Omega,d\nu)$ then $G$ is a refinement of $K$ if and only if $\mu\preceq\nu$. If $\mu\preceq\nu$ then $G$ is a nontrivial refinement of $K$ if and only if $\nu(\Omega)>\mu(\Omega)$.

\begin{theorem}\label{chracterizationintergal1}
Let $K,G$ be given by (\ref{kernelkintegral1}) and (\ref{kernelgintegral1}). If $\span\{\phi(x,\cdot)\xi:x\in X,\ \xi\in \Lambda\}$ is dense in both $L^2(\Omega,\cW,d\mu)$ and $L^2(\Omega,\cW,d\nu)$ then $\cH_K\preceq\cH_G$ if and only if $\mu\preceq\nu$. In this case, the refinement $G$ of $K$ is nontrivial if and only if $\nu(\Omega)-\mu(\Omega)>0$.
\end{theorem}
\begin{proof}
When $\cW=\bC$, as a direct consequence of Theorem 8 in \cite{XZ2}, $\cH_{\tilde{K}}\preceq\cH_{\tilde{G}}$ if and only if $\mu\preceq\nu$. The result hence follows from Proposition \ref{thirdcharacterization}. When $\cW$ is a general Hilbert space, it can be proved by arguments similar to those in \cite{XZ2}.
\end{proof}

\subsection{Scalar-valued kernels with respect to operator-valued measures.}

Again, $\cB$ is a Banach space and $\cF$ denotes a $\sigma$-algebra consisting of subsets of a fixed set $\Omega$. A $\cB$-valued measure on $\cF$ is a function from $\cF$ to $\cB$ that is countably additive in the sense that for every sequence of pairwise disjoint sets $E_j\in\cF$, $j\in\bN$
$$
\mu\biggl(\bigcup_{j=1}^\infty E_j\biggr)=\sum_{j=1}^\infty \mu(E_j),
$$
where the series converges in the norm of $\cB$. Every $\cB$-valued measure $\mu$ on $\cF$ comes with a scalar-valued measure $|\mu|$ on $\cF$ defined by
$$
|\mu|(E):=\sup_{\cP}\sum_{F\in\cP}\|\mu(F)\|_{\cB},\ \ E\in\cF,
$$
where the supremum is taken over all partitions $\cP$ of $E$ into countably many pairwise disjoint members of $\cF$. We call $|\mu|$ the {\it variation} of $\mu$ and shall only work with these vector-valued measures $\mu$ that are of {\it bounded variation}, that is, $|\mu|(\Omega)<+\infty$. We note that $\mu$ vanishes on sets of zero $|\mu|$ measure. It implies that $\mu$ is absolutely continuous with respect to $|\mu|$ in the sense that
$$
\lim_{|\mu(E)|\to0}\mu(E)=0.
$$

The only type of integration that we shall need is to integrate a bounded $\cF$-measurable scalar-valued function with respect to a $\cB$-valued measure of bounded variation. Denote by $L^\infty(\Omega,d|\mu|)$ the Banach space of essentially bounded $\cF$-measurable functions on $\Omega$ with the norm
$$
\|f\|_{L^\infty(\Omega,d|\mu|)}:=\inf\left\{a\ge0:|\mu|(\{|f|>a\})=0\right\}.
$$
For a simple function $f:\Omega\to \bC$ of the form
$$
f=\sum_{j=1}^n \alpha_j\chi_{E_j},
$$
where $\alpha_j\in\bC$ and $E_j$ are pairwise disjoint members in $\cF$, we define
$$
\int_E fd\mu:=\sum_{j=1}^n \alpha_j \mu(E_j\cap E),\ \ E\in\cF.
$$
Clearly,
$$
\left\|\int_E fd\mu\right\|_{\cB}\le \|f\|_{L^\infty(\Omega,d|\mu|)}|\mu|(E).
$$
Therefore, the map sending a simple function $f$ to $\int_E fd\mu$ can be uniquely extended to a bounded linear operator from $L^\infty(\Omega,d|\mu|)$ to $\cB$. The outcome of the application of the resulting operator on a general $f\in L^\infty(\Omega,d|\mu|)$ is still denoted by $\int_E fd\mu$. This is how the $\cB$-valued integral is defined.

It is time to present the second type of reproducing kernels defined by integration:
\begin{equation}\label{generalkernelintegral2}
K(x,y):=\int_\Omega \Psi(x,y,t)d\mu(t),\ \ x,y\in X,
\end{equation}
where $\mu$ is an $\cL_+(\Lambda)$-valued measure on $\cF$ of bounded variation, and $\Psi$ is a scalar-valued function such that $\Psi(\cdot,\cdot,t)$ is a scalar-valued reproducing kernel on $X$ for all $t\in\Omega$ and for all $x,y\in X$, $\Psi(x,y,\cdot)$ is bounded and $\cF$-measurable. We verify that (\ref{generalkernelintegral2}) indeed defines an $\cL(\Lambda)$-valued reproducing kernel.

\begin{prop}\label{generalkernelis}
With the above assumptions on $\Psi$ and $\mu$, the function $K$ defined by (\ref{generalkernelintegral2}) is an $\cL(\Lambda)$-valued reproducing kernel on $X$.
\end{prop}
\begin{proof}
Fix finite $x_j\in X$ and $\xi_j\in \Lambda$, $j\in\bN_n$. For any $\varepsilon>0$, there exist simple functions
$$
f_{j,k}:=\sum_{l=1}^m \alpha_{j,k,l}\chi_{E_l},\ \ j,k\in\bN_n
$$
such that
\begin{equation}\label{generalkerneliseq1}
\|\Psi(x_j,x_k,\cdot)-f_{j,k}\|_{L^\infty(\Omega,d|\mu|)}<\varepsilon,\ \ j,k\in\bN_n.
\end{equation}
Here, $\alpha_{j,k,l}\in\bC$ and $E_l$ are pairwise disjoint sets in $\cF$ with $|\mu|(E_l)>0$, $l\in\bN_m$. By (\ref{generalkerneliseq1}) and the definition of integration in this section,
\begin{equation}\label{generalkerneliseq2}
\left|\sum_{j=1}^n\sum_{k=1}^n (K(x_j,x_k)\xi_j,\xi_k)_\Lambda-\sum_{j=1}^n\sum_{k=1}^n\left(\biggl(\int_\Omega f_{j,k}d\mu\biggr)\,\xi_j,\xi_k\right)_\Lambda\right|\le \varepsilon |\mu|(\Omega) \biggl(\sum_{j=1}^n\|\xi_j\|_\Lambda\biggr)^2.
\end{equation}
We may choose by (\ref{generalkerneliseq1}) for each $l\in\bN_m$ some $t_l\in E_l$ such that
$$
\left|\Psi(x_j,x_k,t_l)-\alpha_{j,k,l}\right|\le \varepsilon.
$$
Letting
$$
S:=\sum_{j=1}^n\sum_{k=1}^n\sum_{l=1}^m \Psi(x_j,x_k,t_l)(\mu(E_l)\xi_j,\xi_k)_\Lambda,
$$
we get by the above equation that
\begin{equation}\label{generalkerneliseq3}
\begin{array}{l}
\displaystyle{\left|\sum_{j=1}^n\sum_{k=1}^n\left(\biggl(\int_\Omega f_{j,k}d\mu\biggr)\,\xi_j,\xi_k\right)_\Lambda-S\right|}\displaystyle{\le \left|\sum_{j=1}^n\sum_{k=1}^n\sum_{l=1}^m |\alpha_{j,k,l}-\Psi(x_j,x_k,t_l)|(\mu(E_l)\xi_j,\xi_k)_\Lambda\right|}\\
\quad\quad\quad\quad\quad\quad\quad\quad\quad\quad\displaystyle{\le \varepsilon\sum_{j=1}^n\sum_{k=1}^n\sum_{l=1}^m \|\mu(E_l)\|_{\cL(\Lambda)}\|\xi_j\|_\Lambda\|\xi_k\|_\Lambda\le \varepsilon |\mu|(\Omega) \biggl(\sum_{j=1}^n\|\xi_j\|_\Lambda\biggr)^2. }
\end{array}
\end{equation}
Combining (\ref{generalkerneliseq2}) and (\ref{generalkerneliseq3}) yields that
\begin{equation}\label{generalkerneliseq4}
\left|\sum_{j=1}^n\sum_{k=1}^n (K(x_j,x_k)\xi_j,\xi_k)_\Lambda-S\right|\le 2\varepsilon |\mu|(\Omega) \biggl(\sum_{j=1}^n\|\xi_j\|_\Lambda\biggr)^2.
\end{equation}
Since $\Psi(\cdot,\cdot,t_l)$ is a scalar-valued reproducing kernel on $X$, $[\Psi(x_j,x_k,t_l):j,k\in\bN_n]$ is a positive semi-definite matrix for each $l\in\bN_m$. So are $[(\mu(E_l)\xi_j,\xi_k)_\Lambda:j,k\in\bN_n]$, $l\in\bN_m$ as $\mu(E_l)\in\cL_+(\Lambda)$. By the Schur product theorem (see, for example, \cite{HornJohnson}, page 309), the Hadamard product of two positive semi-definite matrices remains positive semi-definite. We obtain by this fact that $S>0$, which together with (\ref{generalkerneliseq4}), and the fact that $\varepsilon$ can be arbitrarily small, proves (\ref{positivity}).
\end{proof}

To investigate the refinement relationship, we shall consider a simplified version of (\ref{generalkernelintegral2}) that covers a large class of operator-valued reproducing kernels. Let $\phi:X\times \Omega\to\bC$ be such that $\phi(x,\cdot)$ is a bounded $\cF$-measurable function for every $x\in X$ and such that
\begin{equation}\label{densitycondition}
\overline{\span}\{\phi(x,\cdot):x\in X\}=L^2(\Omega,d\gamma)\mbox{ for any finite nonnegative measure }\gamma\mbox{ on }\cF.
\end{equation}
We shall see by the concrete examples in the next section that the denseness requirement (\ref{densitycondition}) is not too restricted in applications. The kernels we shall consider are
\begin{equation}\label{kernelkintegral2}
K(x,y):=\int_\Omega \phi(x,t)\overline{\phi(y,t)}d\mu(t),\ \ x,y\in X
\end{equation}
and
\begin{equation}\label{kernelgintegral2}
G(x,y):=\int_\Omega \phi(x,t)\overline{\phi(y,t)}d\nu(t),\ \ x,y\in X,
\end{equation}
where $\mu,\nu$ are two $\cL_+(\Lambda)$-valued measures on $\cF$ of bounded variation. By Proposition \ref{generalkernelis}, $K,G$ are $\cL(\Lambda)$-valued reproducing kernels on $X$. Our idea is to use the Radon-Nikodym property of vector-valued measures to study the refinement property.

Let $\cB$ be a Banach space and $\gamma$ a finite nonnegative measure on $\cF$. We say that a $\cB$-valued measure $\rho$ on $\cF$ of bounded variation has the {\it Radon-Nikodym property} with respect to $\gamma$ if there is a $\gamma$-Bochner integrable function $\Gamma:\Omega\to\cL_+(\Lambda)$ such that for all $E\in\cF$
$$
\rho(E)=\int_E \Gamma d\gamma.
$$
Apparently, this could only be true when $\rho$ is absolutely continuous with respect to $\gamma$. For this reason, we also say that the space $\cB$ has the Radon-Nikodym property with respect to $\gamma$ if every $\cB$-valued measure of bounded variation that is absolutely continuous with respect to $\gamma$ has the Radon-Nikodym property with respect to $\gamma$. Moreover, $\cB$ is said to have the Radon-Nikodym property if it has it with respect to any finite nonnegative measure on any measure space $\cF$.

Strikingly different from the scalar-valued case, a Banach space $\cB$ may not have the Radon-Nikodym property. For instance, the Banach space $c_0$ of all sequences $\alpha:=(\alpha_j\in\bC:j\in\bN)$ with
$$
\lim_{j\to\infty}|\alpha_j|=0
$$
under the norm $\|\alpha\|_{c_0}:=\sup\{|\alpha_j|:j\in\bN\}$ does not have the property with respect to the Lebesgue measure (see, \cite{DU}, page 60). Consequently, the space $\cL(\Lambda)$ does not have the Radon-Nikodym property when $\Lambda$ is infinite-dimensional. To see this, since $\Lambda$ is separable we let $\{e_j:j\in\bN\}$ be an orthonormal basis for $\Lambda$. Denote by $\cL_0(\Lambda)$ the set of all the operators $T\in \cL(\Lambda)$ such that
$$
Te_j=\alpha_j e_j,\ \ j\in\bN
$$
for some $\alpha\in c_0$. One sees that $\|T\|_{\cL(\Lambda)}=\|\alpha\|_{c_0}$, \cite{Conway}. As a result, $\cL_0(\Lambda)$ is a closed subspace of $\cL(\Lambda)$ that is isometrically isomorphic to $c_0$. Since $c_0$ does not have the Radon-Nikodym property, neither does $\cL_0(\Lambda)$. A Banach space has the Radon-Nikodym property if and only if each of its closed linear subspaces does \cite{DU}. By this fact, $\cL(\Lambda)$ does not have Radon-Nikodym property.

We shall focus on the situation where this desired property holds. For example, reflexive Banach spaces have the Radon-Nikodym property \cite{DU}. In applications, $\Lambda$ is usually finite-dimensional. In this case, $\cL(\Lambda)$ is of finite dimension as well. Any two norms on a finite-dimensional Banach space are equivalent and a finite-dimensional $\cL(\Lambda)$ can be endowed with a norm that makes it a Hilbert space. It yields that  $\cL(\Lambda)$ is reflexive. The conclusion is that when $\Lambda$ is finite-dimensional, $\cL(\Lambda)$ does have the Radon-Nikodym property. Another way of overcoming the difficulty is to confine to a subclass of $\cL(\Lambda)$, for example, to the Schatten class \cite{BS}. Denote for each compact operator $T\in\cL(\Lambda)$ by $s_j(T)$, $j\in\bN$, the nonnegative square root of the $j$-th largest eigenvalue of $T^*T$. It is called the $j$-th {\it singular number} of $T$. For $p\in(1,+\infty)$, the $p$-th Schatten class $\cS_p(\Lambda)$ consists of all the compact linear operators $T\in\cL(\Lambda)$ with the norm
$$
\|T\|_{\cS_p(\Lambda)}:=\biggl(\sum_{j=1}^\infty (s_j(T))^p\biggr)^{1/p}<+\infty.
$$
The $p$-th Schatten class $S_p(\Lambda)$ is a reflexive Banach space and hence has the Radon-Nikodym property. When $p=2$, $S_2(\Lambda)$ is the class of Hilbert-Schmidt operators and
$$
\|T\|_{\cS_2(\Lambda)}=\biggl(\sum_{j=1}^\infty \|Te_j\|_{\Lambda}\biggr)^{1/2}.
$$
We shall not go into further details about the Radon-Nikodym property. Interested readers are referred to Chapter III of \cite{DU} and the references therein.

The assumption we shall need is that there exists a finite nonnegative measure $\gamma$ on $\cF$ such that both $\mu$ and $\nu$ have the Radon-Nikodym property with respect to $\gamma$. In other words, there exist $\gamma$-Bochner integrable functions $\Gamma_\mu,\Gamma_\nu:\Omega\to\cL_+(\Lambda)$ such that
\begin{equation}\label{radonnikodym}
\mu(E)=\int_E \Gamma_\mu d\gamma\quad\mbox{ and }\quad\nu(E)=\int_E \Gamma_\nu d\gamma\quad\mbox{ for all }E\in\cF.
\end{equation}
Such two functions exist if $\gamma:=|\mu|+|\nu|$ and $\mu,\nu$ take values in the $p$-th Schatten class of $\cL(\Lambda)$, $1<p<+\infty$.

Suppose that $K,G$ are given by (\ref{kernelkintegral2}) and (\ref{kernelgintegral2}), where $\phi,\mu,\nu$ satisfy (\ref{densitycondition}) and (\ref{radonnikodym}). Our purpose is to investigate $\cH_K\preceq\cH_G$. To this end, let us first identify $\cH_{\tilde{K}}$ and $\cH_{\tG}$. We shall only present results for $\cH_{\tK}$ as those for $\cH_{\tG}$ have a similar form.

\begin{lemma}\label{chracterizationintergal2lemma1}
The RKHS $\cH_{\tK}$ consists of functions $F_f$ of the form
$$
F_f(x,\xi):=\int_\Omega (\Gamma_\mu(t) f(t),\xi)_\Lambda\overline{\phi(x,t)}d\gamma(t),\ \ x\in X,\ \xi\in\Lambda,
$$
where $f$ can be an arbitrary element from the Hilbert space $\cW_\mu$ of $\gamma$-measurable functions from $\Omega$ to $\Lambda$ such that
$$
\|f\|_{\cW_\mu}:=\left(\int_\Omega (\Gamma_\mu (t) f(t),f(t))_\Lambda d\gamma(t)\right)^{1/2}<+\infty.
$$
Moreover, $\|F_f\|_{\cH_{\tK}}=\|f\|_{\cW_\mu}$ for all $f\in \cW_\mu$.
\end{lemma}
\begin{proof}
We observe for all $x,y\in X$ and $\xi,\eta\in\Lambda$ that
$$
\tilde{K}((x,\xi),(y,\eta))=\int_\Omega\phi(x,t)\overline{\phi(y,t)}(\Gamma_\mu(t) \xi,\eta)_\Lambda d\gamma(t).
$$
Thus, we may choose $\cW_\mu$ as a feature space for $\tK$. The associated feature map $\Phi_\mu:X\times\Lambda\to\cW_\mu$ is then selected as
$$
\Phi_\mu(x,\xi)(t):=\phi(x,t)\xi,\ \ t\in\Omega.
$$
We next verify the denseness condition that $\overline{\span}\{\Phi_\mu(x,\xi):\ x\in X,\ \xi\in\Lambda\}=\cW_\mu$. Suppose that $f\in\cW_\mu$ is orthogonal to $\Phi_\mu(x,\xi)$ for all $x\in X$ and $\xi\in\Lambda$, that is,
$$
\int_\Omega (\Gamma_\mu f(t),\xi)_\Lambda\overline{\phi(x,t)}d\gamma(t)=0\mbox{ for all }x\in X,\ \xi\in\Lambda.
$$
By (\ref{densitycondition}),
$$
(\Gamma_\mu(t)f(t),\xi)_\Lambda=0\ \gamma-\ae
$$
As this holds for an arbitrary $\xi\in\Lambda$, $\Gamma_\mu(t)f(t)=0$ $\gamma-\ae$ It implies that $\|f\|_{\cW_\mu}=0$. The result now follows immediately from Lemma \ref{rkhsfeature}.
\end{proof}

For two operators $A,B\in \cL_+(\Lambda)$, we write $A\preceq B$ if for all $\xi\in\Lambda$ there exists some $\eta\in \Lambda$ such that
$$
A\xi=B\eta\mbox{ and }(A\xi,\xi)_\Lambda=(B\eta,\eta)_\Lambda.
$$
We make a simple observation about this special relationship between two linear operators.

Let $\ker(A)$ and $\ran(A)$ be the kernel and range of $A$, respectively. If $\ran(A)$ is closed then as $A$ is self-adjoint, there holds the direct sum decomposition
\begin{equation}\label{directsum}
\Lambda=\ker(A)\oplus \ran(A).
\end{equation}
Thus, $A$ is bijective and bounded from $\ran(A)$ to $\ran(A)$. By the open mapping theorem, it has a bounded inverse on $\ran(A)$, which we denote by $A^{-1}$.

\begin{prop}\label{chracterizationintergal2prop}
Suppose that $A,B\in \cL_+(\Lambda)$ have closed range. Then $A\preceq B$ if and only if
\begin{equation}\label{chracterizationintergal2propeq1}
\ran(A)\subseteq \ran(B)
\end{equation}
and
\begin{equation}\label{chracterizationintergal2propeq2}
P_{B,A}B^{-1}=A^{-1}\mbox{ on }\ran(A),
\end{equation}
where $P_{B,A}$ denotes the orthogonal projection from $\ran(B)$ to $\ran(A)$. Particularly, if $A$ is onto then $A\preceq B$ if and only if $A=B$.
\end{prop}
\begin{proof}
Let $A,B$ have closed range. Suppose first that $A\preceq B$. Then (\ref{chracterizationintergal2propeq1}) clearly holds true. Set for each $\xi\in\ran(A)$
$$
\eta_\xi:=B^{-1}A\xi.
$$
Clearly, the mapping $\xi\to\eta_\xi$ is linear from $\ran(A)$ to $\ran(B)$. Thus, we have for arbitrary $\xi,\xi'\in\Lambda$ that
$$
(A\xi'+A\xi,\xi'+\xi)_\Lambda=(B\eta_{\xi'+\xi},\eta_{\xi'+\xi})_\Lambda=(B\eta_{\xi'}+B\eta_\xi,\eta_{\xi'}+\eta_\xi)_\Lambda,
$$
which implies that
$$
\Re(A\xi',\xi)_\Lambda=\Re(B\eta_{\xi'},\eta_\xi)_\Lambda.
$$
A textbook trick yields that for all $\xi,\xi'\in \ran(A)$,
$$
(A\xi',\xi)_\Lambda=(B\eta_{\xi'},\eta_\xi)_\Lambda=(A\xi',\eta_{\xi})_\Lambda.
$$
We hence obtain that $\xi-\eta_{\xi}\in\ker(A)$ for all $\xi\in\ran(A)$. Consequently,
$$
A\xi-AB^{-1}A\xi=A\xi-A\eta_{\xi}=0\mbox{ for all }\xi\in\ran(A),
$$
from which (\ref{chracterizationintergal2propeq2}) follows.

On the other hand, suppose that (\ref{chracterizationintergal2propeq1}) and (\ref{chracterizationintergal2propeq2}) hold true. Then we choose for each $\xi\in \Lambda$
$$
\eta:=B^{-1}A\xi
$$
and verify that $B\eta=A\xi$ and
$$
(B\eta,\eta)_\Lambda=(A\xi,B^{-1}A\xi)_\Lambda=(A\xi,P_{B,A}B^{-1}A\xi)_\Lambda=(A\xi,A^{-1}A\xi)_\Lambda=(A\xi,\xi)_\Lambda.
$$

Finally, if $A$ is onto then by (\ref{chracterizationintergal2propeq1}), $\ran(A)=\ran(B)=\Lambda$. According to (\ref{directsum}), both $A$ and $B$ are injective. Therefore, they possess a bounded inverse on $\Lambda$. It implies that $P_{B,A}$ is the identity operator on $\Lambda$. By equation (\ref{chracterizationintergal2propeq2}), $A=B$. The proof is complete.
\end{proof}

We are ready to present the main result of this section.

\begin{theorem}\label{chracterizationintergal2}
Let $K,G$ be given by (\ref{kernelkintegral2}) and (\ref{kernelgintegral2}), where $\phi,\mu,\nu$ satisfy (\ref{densitycondition}) and (\ref{radonnikodym}). Then $\cH_K\preceq\cH_G$ if and only if $\Gamma_\mu\preceq \Gamma_\nu$ $\gamma-\ae$
\end{theorem}
\begin{proof}
By Proposition \ref{thirdcharacterization} and Lemma \ref{chracterizationintergal2lemma1}, $\cH_K\preceq\cH_G$ if and only if for all $f\in\cW_\mu$, there exists some $g\in\cW_\nu$ such that
\begin{equation}\label{chracterizationintergal2eq1}
\int_\Omega (\Gamma_\mu(t)f(t),\xi)_\Lambda\overline{\phi(x,t)}d\gamma(t)=\int_\Omega (\Gamma_\nu(t)g(t),\xi)_\Lambda\overline{\phi(x,t)}d\gamma(t)\mbox{ for all }x\in X,\ \xi\in\Lambda
\end{equation}
and
\begin{equation}\label{chracterizationintergal2eq2}
\int_\Omega (\Gamma_\mu(t)f(t),f(t))_\Lambda d\gamma(t)=\int_\Omega (\Gamma_\nu(t)g(t),g(t))_\Lambda d\gamma(t).
\end{equation}
By the denseness condition (\ref{densitycondition}), (\ref{chracterizationintergal2eq1}) holds true if and only if
$$
(\Gamma_\mu(t)f(t),\xi)_\Lambda=(\Gamma_\nu(t)g(t),\xi)_\Lambda\ \mbox{ for }\gamma-\ae t\in\Omega\mbox{ and all }\xi\in\Lambda,
$$
which is equivalent to
\begin{equation}\label{chracterizationintergal2eq3}
\Gamma_\mu(t)f(t)=\Gamma_\nu(t)g(t)\mbox{ for }\gamma-\ae t\in\Omega.
\end{equation}
We conclude that $\cH_K\preceq\cH_G$ if and only if for every $f\in\cW_\mu$, there exists some $g\in\cW_\nu$ such that equations (\ref{chracterizationintergal2eq2}) and (\ref{chracterizationintergal2eq3}) hold true.

Suppose that $\Gamma_\mu\preceq\Gamma_\nu$ $\gamma-\ae$ Then clearly, for each $f\in\cW_\mu$, we can find a function $g:\Omega\to\Lambda$ which is defined $\gamma$-almost everywhere and satisfies (\ref{chracterizationintergal2eq3}) and
$$
(\Gamma_\mu(t)f(t),f(t))_\Lambda=(\Gamma_\nu(t)g(t),g(t))_\Lambda\mbox{ for }\gamma-\ae t\in\Omega.
$$
The above equation implies (\ref{chracterizationintergal2eq2}). Therefore, $\cH_K\preceq\cH_G$.

On the other hand, suppose that we can find for every $f\in\cW_\mu$ some $g_f\in\cW_\nu$ satisfying (\ref{chracterizationintergal2eq2}) and (\ref{chracterizationintergal2eq3}). The function $g_f$ can be chosen so that $f\to g_f$ is linear from $\cW_\mu$ to $\cW_\nu$. A trick similar to that used in Lemma \ref{chracterizationintergal2lemma1} enables us to obtain from (\ref{chracterizationintergal2eq2}) and (\ref{chracterizationintergal2eq3}) that
$$
\int_\Omega(\Gamma_\mu(t) f'(t),f(t)-g_f(t))_\Lambda d\gamma(t)=0\mbox{ for all }f'\in\cW_\mu.
$$
Letting $f':=\phi(x,\cdot)\xi$ for arbitrary $x\in X$ and $\xi\in\Lambda$ in the above equation and invoking (\ref{densitycondition}), we have that
$$
\Gamma_\mu(t) (f(t)-g_f(t))=0\mbox{ for }\gamma-\ae t\in\Omega.
$$
By the above equation and (\ref{chracterizationintergal2eq3}), we get for $\gamma$-almost every $t\in\Omega$ that
$$
(\Gamma_\nu(t) g_f(t),g_f(t))_\Lambda=(\Gamma_\mu(t)f(t),g_f(t))_\Lambda=(f(t),\Gamma_\mu(t)g_f(t))_\Lambda=(f(t),\Gamma_\mu(t)f(t))_\Lambda=(\Gamma_\mu(t)f(t),f(t))_\Lambda.
$$
Since (\ref{chracterizationintergal2eq3}) and the above equation are true for an arbitrary $f\in\cW_\mu$, $\Gamma_\mu\preceq\Gamma_\nu$ $\gamma-\ae$
\end{proof}

\section{Examples}
\setcounter{equation}{0}

We present in this section several concrete examples of refinement of operator-valued reproducing kernels. They are built on the general characterizations established in the last two sections.

\subsection{Translation invariant reproducing kernels}

Let $d\in\bN$ and $K$ be an $\cL(\Lambda)$-valued reproducing kernel on $\bR^d$. We say that $K$ is {\it translation invariant} if for all $x,y,a\in\bR^d$
$$
K(x-a,y-a)=K(x,y).
$$
A celebrated characterization due to Bochner \cite{Bochner2} states that every continuous scalar-valued translation invariant reproducing kernel on $\bR^d$ must be the Fourier transform of a finite nonnegative Borel measure on $\bR^d$, and vice versa. This result has been generalized to the operator-valued case \cite{Berberian,CDTU,Fillmore}. Specifically, a continuous function $K$ from $\bR^d\times\bR^d$ to $\cL(\Lambda)$ is a translation invariant reproducing kernel if and only if it has the form
\begin{equation}\label{translationinavriantkernelk}
K(x,y)=\int_{\bR^d}e^{i(x-y)\cdot t}d\mu(t),\ \ x,y\in\bR^d,
\end{equation}
for some $\mu\in\cB(\bR^d,\Lambda)$, the set of all the $\cL_+(\Lambda)$-valued measures of bounded variation on the $\sigma$-algebra of Borel subsets in $\bR^d$. Let $G$ be the kernel given by
\begin{equation}\label{translationinavriantkernelg}
G(x,y)=\int_{\bR^d}e^{i(x-y)\cdot t}d\nu(t),\ \ x,y\in\bR^d,
\end{equation}
where $\nu\in\cB(\bR^d,\Lambda)$. The purpose of this subsection is to characterize $\cH_K\preceq\cH_G$ in terms of $\mu,\nu$. To this end, we first investigate the structure of the RKHS of a translation invariant $\cL(\Lambda)$-valued reproducing kernel.

Let $\gamma$ be an arbitrary measure in $\cB(\bR^d,\Lambda)$ and $L$ the associated translation invariant reproducing kernel defined by
\begin{equation}\label{kernelL}
L(x,y)=\int_{\bR^d}e^{i(x-y)\cdot t}d\gamma(t),\ \ x,y\in\bR^d.
\end{equation}
There exists a decomposition of $\gamma$ with respect to the Lebesgue measure $dx$ on $\bR^d$ \cite{DU} as follows:
$$
\gamma=\gamma_c+\gamma_s,
$$
where $\gamma_c,\gamma_s$ are the unique measures in $\cB(\bR^d,\Lambda)$ such that $\gamma_c$ is absolutely continuous with respect to $dx$, and for each continuous linear functional $\lambda$ on $\cL(\Lambda)$, the scalar-valued measure $\lambda\gamma_s$ and $dx$ are mutually singular. It follows from this decomposition of measures a decomposition of $L$:
$$
L=L_c+L_s,
$$
where
\begin{equation}\label{kernellcls}
L_c(x,y)=\int_{\bR^d}e^{i(x-y)\cdot t}d\gamma_c(t),\ \ L_s(x,y)=\int_{\bR^d}e^{i(x-y)\cdot t}d\gamma_s(t),\ \ \ x,y\in\bR^d.
\end{equation}
Our first observation is that $\cH_L$ is the orthogonal direct sum of $\cH_{L_c}$ and $\cH_{L_s}$. Two lemmas are needed to prove this useful fact.

\begin{lemma}\label{orthogonaldirectsumlemma1}
Let $L_c,L_s$ be given by (\ref{kernellcls}). Then for all $\xi\in \Lambda$ and $x,y\in\bR^d$
\begin{equation}\label{orthogonaldirectsumlemma1eq}
(L_a(x,y)\xi,\xi)_\Lambda=\int_{\bR^d}e^{i(x-y)\cdot t}d\gamma_{a,\xi}(t),\ \ a=c\mbox{ or }s,
\end{equation}
where $\gamma_{a,\xi}$ is a scalar-valued Borel measure on $\bR^d$ defined for each Borel set $E\subseteq\bR^d$ by
$$
\gamma_{a,\xi}(E):=(\gamma_a(E)\xi,\xi)_\Lambda,\ \ a=c\mbox{ or }s.
$$
\end{lemma}
\begin{proof} Let $a\in\{c,s\}$, $\xi\in\Lambda$, $x,y\in \bR^d$, and $s_n$ be a sequence of simple functions on $\bR^d$ that converges to $e^{i(x-y)\cdot t}$ in $L^\infty(\bR^d,dx)$. Then
$$
\lim_{n\to\infty}\left(\biggl(\int_{\bR^d}s_nd\gamma_a\biggr)\, \xi,\xi\right)_\Lambda=(L_a(x,y)\xi,\xi)_\Lambda.
$$
By definition, we have for each $n\in\bN$ that
$$
\lim_{n\to\infty}\left(\biggl(\int_{\bR^d}s_nd\gamma_a\biggr)\, \xi,\xi\right)_\Lambda=\int_{\bR^d}s_n d\gamma_{a,\xi}.
$$
As
$$
\lim_{n\to\infty}\int_{\bR^d}s_n d\gamma_{a,\xi}=\int_{\bR^d}e^{i(x-y)\cdot t}d\gamma_{a,\xi}(t),
$$
we conclude from the previous two equations that (\ref{orthogonaldirectsumlemma1eq}) holds true.
\end{proof}

\begin{lemma}\label{zerointersection}
There holds $\cH_{L_c}\cap\cH_{L_s}=\{0\}$.
\end{lemma}
\begin{proof}
We introduce for each $\xi\in \Lambda$ two scalar-valued translation invariant reproducing kernels on $\bR^d$ by setting
$$
A_a(x,y):=(L_a(x,y)\xi,\xi)_\Lambda,\ \ x,y\in\bR^d,\ \ a\in\{c,s\}.
$$
By Lemma \ref{orthogonaldirectsumlemma1}, we have the alternative representations for $A_c$ and $A_s$
$$
A_a(x,y)=\int_{\bR^d}e^{i(x-y)\cdot t}d\gamma_{a,\xi}(t),\ \ x,y\in\bR^d,\ \ a=c\mbox{ or }s.
$$
By the Lebesgue decomposition of $\gamma$, $\gamma_{c,\xi}$ is absolutely continuous with respect to $dx$ while $\gamma_{s,\xi}$ and $dx$ are mutually singular. As a consequence, $\cH_{A_c}\cap\cH_{A_s}=\{0\}$ by Lemma 17 in \cite{XZ2}.

Let $a\in\{c,s\}$. By (\ref{reproducing}),
$$
A_a(x,y)=(L_a(x,\cdot)\xi,L_a(y,\cdot)\xi)_{\cH_{L_a}},\ \ x,y\in\bR^d.
$$
A feature map for $A_a$ may hence be chosen as
$$
\Phi_a(x):=L_a(x,\cdot)\xi,\ \ x\in\bR^d
$$
with the feature space being $\cH_{L_a}$. We identify by Lemma \ref{rkhsfeature} that
\begin{equation}\label{zerointersectioneq1}
\cH_{A_a}=\{(\tilde{f}(\cdot),\xi)_\Lambda:\tilde{f}\in\cH_{L_a}\}.
\end{equation}

Assume that $\cH_{L_c}\cap\cH_{L_s}\ne\{0\}$. Then there exist nontrivial functions $\tilde{f}\in\cH_{L_c}$ and $\tilde{g}\in\cH_{L_s}$ such that $\tilde{f}=\tilde{g}$. As a result, there exists some $\xi\in\Lambda$ such that $(\tilde{f}(\cdot),\xi)_\Lambda$ is not the trivial function. By equation (\ref{zerointersectioneq1})
$$
(\tilde{f}(\cdot),\xi)_\Lambda=(\tilde{g}(\cdot),\xi)_\Lambda\in \cH_{A_c}\cap\cH_{A_s},
$$
contradicting the fact that $\cH_{A_c}\cap\cH_{A_s}=\{0\}$.
\end{proof}

\begin{theorem}\label{orthogonaldirectsum}
The space $\cH_L$ is the orthogonal direct sum of $\cH_{L_c}$ and $\cH_{L_s}$, namely, $\cH_L=\cH_{L_c}\bigoplus\cH_{L_s}$.
\end{theorem}
\begin{proof}
The result follows directly from Lemma \ref{zerointersection} and Proposition \ref{firstcharacterization}.
\end{proof}

We are now in a position to study the refinement relationship $\cH_K\preceq\cH_G$, where $K,G$ are defined by (\ref{translationinavriantkernelk}) and (\ref{translationinavriantkernelg}). Firstly, the task can be separated into two related ones according to the Lebesgue decomposition of measures $\mu,\nu$.

\begin{prop}\label{separation}
There holds $\cH_K\preceq\cH_G$ if and only if $\cH_{K_c}\preceq\cH_{G_c}$ and $\cH_{K_s}\preceq\cH_{G_s}$.
\end{prop}
\begin{proof}
By Theorem \ref{orthogonaldirectsum}, $\cH_K=\cH_{K_c}\bigoplus\cH_{K_s}$ and $\cH_G=\cH_{G_c}\bigoplus\cH_{G_s}$. Therefore, if $\cH_{K_c}\preceq\cH_{G_c}$ and $\cH_{K_s}\preceq\cH_{G_s}$ then $\cH_K\preceq\cH_G$.

On the other hand, suppose that $\cH_K\preceq\cH_G$. Let $f\in\cH_{K_c}$. Then $f\in\cH_K$ and $\|f\|_{\cH_{K_c}}=\|f\|_{\cH_K}$. Since $\cH_K\preceq\cH_G$, there exists $g\in\cH_{G_c}$ and $h\in\cH_{G_s}$ such that
$$
f=g+h
$$
and
$$
\|f\|_{\cH_{K_c}}^2=\|f\|_{\cH_K}^2=\|g+h\|_{\cH_G}^2=\|g\|_{\cH_{G_c}}^2+\|h\|_{\cH_{G_s}}^2.
$$
Therefore, to show that $\cH_{K_c}\preceq\cH_{G_c}$ it suffices to show that $h=0$. Assume that $h\ne0$. Note that $f-g\in\cH_{K_c+G_c}$ \cite{Pedrick}, we get that
\begin{equation}\label{separationeq1}
\cH_{K_c+G_c}\cap\cH_{G_s}\ne\{0\}.
\end{equation}
However,
$$
(K_c+G_c)(x,y)=\int_{\bR^d}e^{i(x-y)\cdot t}d(\mu_c+\nu_c)(t),\ \ x,y\in\bR^d
$$
and $\mu_c+\nu_c$ is absolutely continuous with respect to $dx$. Thus, equation (\ref{separationeq1}) contradicts Lemma \ref{zerointersection}. The contradiction proves that $\cH_{K_c}\preceq\cH_{G_c}$. Likewise, one can prove that $\cH_{K_s}\preceq\cH_{G_s}$.
\end{proof}

By Proposition \ref{separation}, we shall study $\cH_{K_c}\preceq\cH_{G_c}$ and $\cH_{K_s}\preceq\cH_{G_s}$ separately. The kernels to be considered are of the following special forms:
\begin{equation}\label{kernelkcgc}
K_c(x,y):=\int_{\bR^d}e^{i(x-y)\cdot t}\varphi_1(t)dt,\ \ G_c(x,y):=\int_{\bR^d}e^{i(x-y)\cdot t}\varphi_2(t)dt,\ \ x,y\in\bR^d
\end{equation}
and
\begin{equation}\label{kernelksgs}
K_s(x,y):=\sum_{j\in\bJ_1}e^{i(x-y)\cdot t_j}A_j,\ \ G_s(x,y):=\sum_{k\in\bJ_2}e^{i(x-y)\cdot t_k}B_k,\ \ x,y\in\bR^d.
\end{equation}
Here, $\varphi_1,\varphi_2$ are two $dx$-Bochner integrable functions from $\bR^d$ to $\cL_+(\Lambda)$, $\{t_j:j\in\bJ_1\}$ and $\{t_k:k\in\bJ_2\}$ are countable sets of pairwise distinct points in $\bR^d$, and $A_j,B_j$ are nonzero operators in $\cL_+(\Lambda)$ such that
$$
\sum_{j\in\bJ_1}\|A_j\|_{\cL(\Lambda)}<+\infty,\ \ \sum_{k\in\bJ_2}\|B_k\|_{\cL(\Lambda)}<+\infty.
$$
The following characterization is a direct consequence of Theorem \ref{chracterizationintergal2}.

\begin{prop}\label{characterizekcgc}
Let $K_c,G_c$ be given by (\ref{kernelkcgc}). Then $\cH_{K_c}\preceq\cH_{G_c}$ if and only if $\varphi_1(t)\preceq\varphi_2(t)$ for almost every $t\in\bR^d$ except for a subset in $\bR^d$ of zero Lebesgue measure.
\end{prop}
\begin{proof}
As $\varphi_1,\varphi_2$ are $dx$-Bochner integrable,
$$
\int_{\bR^d}\|\varphi_j(t)\|_{\cL(\Lambda)}dt<+\infty,\ \ j=1,2.
$$
Define a finite nonnegative Borel measure $\gamma$ on $\bR^d$ by setting for each Borel subset $E$ in $\bR^d$
$$
\gamma(E):=\int_E \|\varphi_1(t)\|_{\cL(\Lambda)}+\|\varphi_2(t)\|_{\cL(\Lambda)}dt.
$$
Evidently, $K_c,G_c$ have the form
$$
K_c(x,y)=\int_{\bR^d}e^{i(x-y)\cdot t}\Gamma_1(t)d\gamma(t),\ \ G_c(x,y)=\int_{\bR^d}e^{i(x-y)\cdot t}\Gamma_2(t)d\gamma(t),\ \ \ x,y\in\bR^d,
$$
where for $j=1,2$,
$$
\Gamma_j(t):=\left\{
\begin{array}{ll}\displaystyle{\frac{\varphi_j(t)}{\|\varphi_1(t)\|_{\cL(\Lambda)}+\|\varphi_2(t)\|_{\cL(\Lambda)}}},& \mbox{ if }\|\varphi_1(t)\|_{\cL(\Lambda)}+\|\varphi_2(t)\|_{\cL(\Lambda)}>0,\\
0,&\mbox{otherwise.}
\end{array}
\right.
$$
It is also clear that $\span\{e^{ix\cdot t}:x\in\bR^d\}$ is dense in $L^2(\bR^d,d\gamma)$. By Theorem \ref{chracterizationintergal2}, $\cH_{K_c}\preceq\cH_{G_c}$ if and only if $\Gamma_1\preceq\Gamma_2$ $\gamma-\ae$ Note that $\Gamma_1(t)\preceq\Gamma_2(t)$ if and only if $\varphi_1(t)\preceq\varphi_2(t)$. If $\varphi_1\preceq\varphi_2$ $dx-\ae$then $\Gamma_1\preceq\Gamma_2$ $\gamma-\ae$as $\gamma$ is absolutely continuous with respect to the Lebesgue measure. On the other hand, suppose that $\Gamma_1\preceq\Gamma_2$ $\gamma-\ae$ Set
$$
E:=\{t\in\bR^d: \|\varphi_1(t)\|_{\cL(\Lambda)}+\|\varphi_2(t)\|_{\cL(\Lambda)}>0\}.
$$
For $t\in E^c$, $\varphi_1(t)=\varphi_2(t)=0$, and thus, $\varphi_1(t)\preceq\varphi_2(t)$. Assume that there exists a Borel subset $F\subseteq\bR^d$ with a positive Lebesgue measure on which $\varphi_1(t)\npreceq\varphi_2(t)$. Then $F\subseteq E$. We reach that $\gamma(F)>0$ and $\Gamma_1(t)\npreceq\Gamma_2(t)$ for $t\in F$, contradicting the fact that $\Gamma_1\preceq\Gamma_2$ $\gamma-\ae$
\end{proof}

For $K_s,G_s$, we have the following result.

\begin{prop}\label{characterizeksgs}
There holds $\cH_{K_s}\preceq\cH_{G_s}$ if and only if
\begin{description}
\item[(1)] $\{t_j:j\in\bJ_1\}\subseteq \{t_k:k\in\bJ_2\}$;

\item[(2)] for each $j\in\bJ_1$, $A_j\preceq B_j$. Here, re-indexing by condition (1) if necessary, we may assume that $\bJ_1\subseteq\bJ_2$.
\end{description}
\end{prop}
\begin{proof}
Introduce a discrete scalar-valued Borel measure $\gamma$ that is supported on $\{t_j:j\in\bJ_1\}\cup\{t_k:k\in\bJ_2\}$ by setting
$$
\gamma(\{t_k\}):=\left\{\begin{array}{ll}
\|A_k\|_{\cL(\Lambda)}+\|B_k\|_{\cL(\Lambda)},& k\in\bJ_1\cap\bJ_2,\\
\|B_k\|_{\cL(\Lambda)},& k\in\bJ_2\setminus \bJ_1,\\
\|A_k\|_{\cL(\Lambda)},& k\in\bJ_1\setminus \bJ_2.
\end{array}
\right.
$$
We also let
$$
\Gamma_A(t_j):=\frac{A_j}{\gamma(\{t_j\})},\ \ j\in\bJ_1\mbox{ and }\Gamma_A(t_k):=\frac{B_k}{\gamma(\{t_k\})},\ \ k\in\bJ_2.
$$
They are discrete $\cL(\Lambda)$-valued functions supported on $\{t_j:j\in\bJ_1\}$ and $\{t_k:k\in\bJ_2\}$, respectively. We reach the following integral representation:
$$
K_s(x,y)=\int_{\bR^d}e^{i(x-y)\cdot t}\Gamma_A(t)d\gamma(t)\mbox{ and }G_s(x,y)=\int_{\bR^d}e^{i(x-y)\cdot t}\Gamma_B(t)d\gamma(t),\ \ x,y\in\bR^d.
$$
By Theorem \ref{chracterizationintergal2}, $\cH_{K_s}\preceq\cH_{G_s}$ if and only if $\Gamma_A\preceq\Gamma_B$ $\gamma-\ae$ It is straightforward to verify that the latter is equivalent to conditions (1)-(2).
\end{proof}

\subsection{Hessian of Scalar-valued reproducing kernels}

Propositions \ref{characterizekcgc} and \ref{characterizeksgs} were established based on Theorem \ref{chracterizationintergal2}. In this subsection, we shall consider special translation invariant reproducing kernels and establish the characterization of refinement using Theorem \ref{chracterizationintergal1}.

Let $k$ be a continuously differentiable translation invariant reproducing kernel on $\bR^d$. We consider the following matrix-valued functions
\begin{equation}\label{hessian1}
K(x,y):=\nabla_{xy}^2k(x,y):=\biggl[\frac{\partial^2k}{\partial_{x_j}\partial_{y_k}}(x,y):j,k\in\bN_d\biggr],\ \ x,y\in\bR^d.
\end{equation}
To ensure that $K$ is an $\cL(\bC^d)$-valued reproducing kernels on $\bR^d$, we make use of the Bochner theorem to get some finite nonnegative Borel measure $\mu$ on $\bR^d$ such that
\begin{equation}\label{hessiank}
k(x,y)=\int_{\bR^d}e^{i(x-y)\cdot t}d\mu(t),\ \ x,y\in\bR^d
\end{equation}
and impose the requirement that
\begin{equation}\label{differentiable}
\int_{\bR^d}tt^Td\mu(t)<+\infty.
\end{equation}
One sees by the Lebesgue dominated convergence theorem that
\begin{equation}\label{hessian1re2}
K(x,y)=\int_{\bR^d}e^{i(x-y)\cdot t} tt^Td\mu(t),\ \ x,y\in\bR^d,
\end{equation}
where we view $t\in\bR^d$ as a $d\times 1$ vector and $t^T$ denotes its transpose $[t_1,t_2,\ldots,t_d]$. By the general integral representation (\ref{kernelkintegral1}) of operator-valued reproducing kernels, $K$ defined by (\ref{hessian1}) is an $\cL(\bC^d)$-valued reproducing kernel on $\bR^d$. Matrix-valued translation invariant reproducing kernels of the form (\ref{hessian1}) are useful for the development of divergence-free kernel methods for solving some special partial differential equations (see, for example, \cite{Lowitzsh,Wendland} and the references therein). Another class of kernels constructed from the Hessian of a scalar-valued translation invariant reproducing kernel is widely applied to the learning of a multivariate function together with its gradient simultaneously \cite{MW,MZ,YC}. Such applications make use of kernels of the form
\begin{equation}\label{hessian2}
\overline{K}(x,y):=\left[
\begin{array}{cc}
k(x,y)&(\nabla_y k(x,y))^*\\
\nabla_xk(x,y)&\nabla^2_{xy}k(x,y)
\end{array}
\right].
\end{equation}
One sees that under condition (\ref{differentiable})
$$
\overline{K}(x,y)=\int_{\bR^d}e^{i(x-y)\cdot t} \rho(t)\rho(t)^*d\mu(t),\ \ x,y\in\bR^d,
$$
where
$$
\rho(t)=[1,it_1,it_2,\ldots,it_d]^T,\ \ t\in\bR^d.
$$
We aim at refining matrix-valued reproducing kernels of the forms (\ref{hessian1}) and (\ref{hessian2}) in this subsection. Specifically, we let $\nu$ be another finite nonnegative Borel measure on $\bR^d$ satisfying
\begin{equation}\label{differentiable2}
\int_{\bR^d}tt^Td\nu(t)<+\infty
\end{equation}
and define for $x,y\in\bR^d$
\begin{equation}\label{hessiang}
g(x,y):=\int_{\bR^d}e^{i(x-y)\cdot t}d\nu(t),\ G(x,y):=\nabla^2_{xy}g(x,y),\ \overline{G}(x,y):=\left[
\begin{array}{cc}
g(x,y)&(\nabla_y g(x,y))^*\\
\nabla_xg(x,y)&\nabla^2_{xy}g(x,y)
\end{array}
\right].
\end{equation}
Our purpose is to characterize $\cH_K\preceq\cH_G$ and $\cH_{\overline{K}}\preceq\cH_{\overline{G}}$ in terms of $k,g$ and $\mu,\nu$.

\begin{theorem}\label{refinehessian}
Let $\mu,\nu$ be finite nonnegative Borel measures on $\bR^d$ satisfying (\ref{differentiable}) and (\ref{differentiable2}), and $k,g$ defined by (\ref{hessiank}) and (\ref{hessiang}). Then $K,G,\overline{K},\overline{G}$ are matrix-valued translation invariant reproducing kernels on $\bR^d$. The four relationships $\cH_K\preceq\cH_G$, $\cH_{\overline{K}}\preceq\cH_{\overline{G}}$, $\cH_k\preceq\cH_g$, and $\mu\preceq\nu$ are equivalent.
\end{theorem}
\begin{proof}
By Theorem \ref{chracterizationintergal1} or a result in \cite{XZ2}, $\cH_k\preceq\cH_g$ if and only if $\mu\preceq\nu$. We shall show by Theorem \ref{chracterizationintergal1} that $\cH_K\preceq\cH_G$ if and only if $\mu\preceq\nu$. The equivalence of $\cH_{\overline{K}}\preceq\cH_{\overline{G}}$ and $\mu\preceq\nu$ can be proved similarly. Set
$$
\phi(x,t):=e^{ix\cdot t}t^T,\ \ x,t\in\bR^d.
$$
Then for each $x,t\in\bR^d$, $\phi(x,t)$ is a linear functional from $\bC^d$ to $\bC$. We observe by (\ref{hessian1re2}) that (\ref{kernelkintegral1}) holds true. So does (\ref{kernelgintegral1}). To apply Theorem \ref{chracterizationintergal1}, it remains to verify that $\span\{\phi(x,\cdot)\xi:x\in\bR^d,\ \xi\in\bC^d\}$ is dense in the Hilbert space $L^2(\bR^d,d\mu)$, which is straightforward. The claim follows immediately from Theorem \ref{chracterizationintergal1}.
\end{proof}

\subsection{Transformation reproducing kernels}

Let us consider a particular class of matrix-valued reproducing kernels whose universality was studied in \cite{CMPY}. The kernels we shall construct are from an input space $X$ to output space $\Lambda=\bC^n$, where $n\in\bN$. To this end, we let $k,g$ be two scalar-valued reproducing kernels on another input space $Y$ and $T_p$ be mappings from $X$ to $Y$, $p\in\bN_n$. Set
\begin{equation}\label{transformationkernel}
K(x,y):=[k(T_px,T_qy):p,q\in\bN_n],\ \ G(x,y):=[g(T_px,T_qy):p,q\in\bN_n],\ \ x,y\in X.
\end{equation}
It is known that $K,G$ defined above are indeed $\cL(\bC^n)$-valued reproducing kernels \cite{CMPY}. This also becomes clear in the proof below. We are interested in the conditions for $\cH_K\preceq\cH_G$ to hold.

\begin{prop}\label{characterizetransformation}
Let $K,G$ be defined by (\ref{transformationkernel}). Then $\cH_K\preceq\cH_G$ if and only if $\cH_{\overline{k}}\preceq\cH_{\overline{g}}$, where $\bar{k},\bar{g}$ are the restriction of $k,g$ on $\cup_{p=1}^n T_p(X)$. In particular, if
\begin{equation}\label{fullmap}
\bigcup_{p=1}^n T_p(X)=Y
\end{equation}
then $\cH_K\preceq\cH_G$ if and only if $\cH_k\preceq\cH_g$.
\end{prop}
\begin{proof}
It is legitimate to assume that (\ref{fullmap}) holds true as otherwise, we may replace $Y$ by $\cup_{p=1}^n T_p(X)$, and $k,g$ by $\bar{k},\bar{g}$, respectively.

Choose arbitrary feature maps and feature spaces $\Phi_1:Y\to \cW_1$ for $k$ and $\Phi_2:Y\to\cW_2$ for $g$ such that
\begin{equation}\label{characterizetransformationeq1}
\overline{\span}\Phi_j(Y)=\cW_j,\ \ j=1,2.
\end{equation}
By Proposition \ref{thirdcharacterization}, $\cH_K\preceq\cH_G$ if and only if $\cH_{\tilde{K}}\preceq\cH_{\tilde{G}}$. We observe for all $x,y\in X$ and $\xi,\eta\in\bC^n$ that
$$
\begin{array}{rl}
\tilde{K}((x,\xi),(y,\eta))&=(K(x,y)\xi,\eta)_{\bC^n}=\displaystyle{\sum_{p=1}^n\sum_{q=1}^n \xi_p\overline{\eta_q}k(T_px,T_qy)}\\
&\displaystyle{=\sum_{p=1}^n\sum_{q=1}^n\xi_p\overline{\eta_q}(\Phi_1(T_px),\Phi_1(T_qy))_{\cW_1}}\\
&\displaystyle{=\biggl(\sum_{p=1}^n\xi_p\Phi_1(T_px),\sum_{q=1}^n\eta_q\Phi_1(T_qy)\biggr)_{\cW_1}}.
\end{array}
$$
Thus, $\tilde{\Phi}_1:X\times\bC^n\to \cW_1$ defined by
$$
\tilde{\Phi}_1(x,\xi):=\sum_{p=1}^n \xi_p\Phi_1(T_px),\ \ x\in X,\ \xi\in\bC^n
$$
is a feature map for $\tilde{K}$. We next verify that $\span\{\tilde{\Phi}_1(x,\xi):x\in X,\ \xi\in\bC^n\}$ is dense in $\cW_1$. Assume that $u\in\cW_1$ is orthogonal to this linear span, that is,
$$
\biggl(u,\sum_{p=1}^n\xi_p\Phi_1(T_px)\biggr)_{\cW_1}=0\mbox{ for all }x\in X,\ \xi\in \bC^n.
$$
Then we have $(u,\Phi_1(T_px))_{\cW_1}=0$ for all $x\in X$ and $p\in\bN_n$. It follows from (\ref{fullmap}) and (\ref{characterizetransformationeq1}) that $u=0$. Similar facts hold for $\tilde{G}$.

By Lemma \ref{rkhsfeature}, $\cH_{\tilde{K}}\preceq\cH_{\tilde{G}}$ if and only if for every $u\in\cW_1$, there exists $v\in\cW_2$ such that
\begin{equation}\label{characterizetransformationeq2}
\biggl(u,\sum_{p=1}^n\xi_p\Phi_1(T_px)\biggr)_{\cW_1}=\biggl(v,\sum_{p=1}^n\xi_p\Phi_2(T_px)\biggr)_{\cW_2}\mbox{ for all }x\in X
\end{equation}
and
\begin{equation}\label{equalnorm}
\|u\|_{\cW_1}=\|v\|_{\cW_2}.
\end{equation}
Recall also that $\cH_k\preceq\cH_g$ if and only if for all $u\in\cW_1$ there exists some $v\in\cW_2$ satisfying (\ref{equalnorm}) and
\begin{equation}\label{characterizetransformationeq3}
(u,\Phi_1(y))_{\cW_1}=(v,\Phi_2(y))_{\cW_2}\mbox{ for all }y\in Y.
\end{equation}
Clearly, (\ref{characterizetransformationeq3}) implies (\ref{characterizetransformationeq2}). Conversely, if (\ref{characterizetransformationeq2}) holds true then we get that
$$
(u,\Phi_1(T_px))_{\cW_1}=(v,\Phi_2(T_px))_{\cW_2}\mbox{ for all }x\in X\mbox{ and }p\in\bN_n,
$$
which together with (\ref{fullmap}) implies (\ref{characterizetransformationeq3}). We conclude that $\cH_{\tilde{K}}\preceq\cH_{\tilde{G}}$ if and only if $\cH_k\preceq\cH_g$.
\end{proof}

A more general case of refinement of transformation reproducing kernels is discussed below. It can be proved by arguments similar to those for the previous proposition.

\begin{prop}\label{characterizetransformation2}
Let $T_p,S_p$ be mappings from $X$ to $Y$ and $k,g$ be scalar-valued reproducing kernels on $Y$. Define
$$
K(x,y):=[k(T_px,T_qy):p,q\in\bN_n],\ \ G(x,y):=[g(S_px,S_qy):p,q\in\bN_n],\ \ x,y\in X.
$$
Suppose that for all $p\in\bN_n$, $\span\{k(T_px,\cdot):x\in X\}$ and $\span\{g(S_px,\cdot):x\in X\}$ are dense in $\cH_k$ and $\cH_g$, respectively. Then $\cH_K\preceq\cH_G$ if and only if $\cH_{k_p}\preceq\cH_{g_p}$ for all $p\in\bN_n$, where
$$
k_p(x,y):=k(T_px,T_py),\ \ g_p(x,y):=g(S_px,S_py),\ \ \ x,y\in X.
$$
\end{prop}

\subsection{Finite Hilbert-Schmidt reproducing kernels}

We consider refinement of finite Hilbert-Schmidt reproducing kernels in this subsection. Let $B_j,C_j$ be invertible operators in $\cL_+(\Lambda)$, $n\le m\in\bN$, and $\Psi_j$, $j\in\bN_m$, be scalar-valued reproducing kernels on the input space $X$. Define
\begin{equation}\label{hilbertschmidt1}
K(x,y):=\sum_{j=1}^n B_j \Psi_j(x,y),\ \ G(x,y)=\sum_{j=1}^m C_j\Psi_j(x,y),\ \ \ x,y\in X.
\end{equation}
By the general integral representation (\ref{generalkernelintegral2}) and Proposition \ref{generalkernelis}, $K,G$ above are $\cL(\Lambda)$-valued reproducing kernels on $X$. To ensure that representation (\ref{hilbertschmidt1}) can not be further simplified, we shall work under the assumption that
\begin{equation}\label{exclusive}
\cH_{\Psi_j}\cap\cH_{\overline{\Psi}_j}=\{0\}\mbox{ for all }j\in\bN_m,
\end{equation}
where
$$
\overline{\Psi}_j:=\sum_{k\in\bN_m\setminus\{j\}}\Psi_k.
$$

\begin{theorem}\label{refinehilbertschmidt1}
Let $K,G$ be defined by (\ref{hilbertschmidt1}), where $B_j,C_j\in\cL_+(\Lambda)$ are invertible and $\Psi_j$, $j\in\bN_m$, are scalar-valued reproducing kernels on $X$ satisfying (\ref{exclusive}). Then $\cH_K\preceq\cH_G$ if and only if $B_j=C_j$, $j\in\bN_n$.
\end{theorem}
\begin{proof}
We first find a feature map for $\tilde{K}$ and $\tilde{G}$. Let $\phi_j:X\to \cW_j$ be an arbitrary feature map for $\Psi_j$ such that $\span\phi_j(X)$ is dense in $\cW_j$, and denote by $\Lambda\otimes\cW_j$ the tensor product of Hilbert spaces $\Lambda$ and $\cW_j$, $j\in\bN_m$. The space $\Lambda\otimes\cW_j$ is a Hilbert space with the inner product
$$
(\xi\otimes u,\eta\otimes v)_{\Lambda\otimes\cW_j}:=(\xi,\eta)_\Lambda(u,v)_{\cW_j},\ \ \xi,\eta\in\Lambda,\ u,v\in\cW_j.
$$
Set $\cW$ the orthogonal direct sum of $\Lambda\otimes\cW_j$, $j\in\bN_n$, whose inner product is defined by
$$
((\xi_j\otimes u_j:j\in\bN_n),(\eta_j\otimes v_j:j\in\bN_n))_{\cW}:=\sum_{j=1}^n(\xi_j,\eta_j)_\Lambda(u_j,v_j)_{\cW_j},\ \ \xi_j,\eta_j\in\Lambda,\ u_j,v_j\in\cW_j,\ \ j\in\bN_n.
$$
We claim that $\Phi:X\times \Lambda\to\cW$ defined by
$$
\Phi(x,\xi):=(\sqrt{B_j}\xi\otimes\phi_j(x):j\in\bN_n),\ \ x\in X,\ \xi\in\Lambda
$$
is a feature map for $\tilde{K}$. Here, $\sqrt{B_j}$ denotes the unique operator $A$ in $\cL_+(\Lambda)$ such that $A^2=B_j$. We verify for all $x,y\in X$ and $\xi,\eta\in\Lambda$ that
$$
\begin{array}{rl}
(\Phi(x,\xi),\Phi(y,\eta))_\cW&\displaystyle{=\sum_{j=1}^n(\sqrt{B_j}\xi,\sqrt{B_j}\eta)_\Lambda(\phi_j(x),\phi_j(y))_{\cW_j}}=\sum_{j=1}^n(B_j\xi,\eta)_\Lambda \Psi_j(x,y)\\
&\displaystyle{=(K(x,y)\xi,\eta)=\tilde{K}((x,\xi),(y,\eta))}.
\end{array}
$$
We next show that the denseness condition
\begin{equation}\label{refinehilbertschmidt1eq1}
\overline{\span}\{\phi(x,\xi):\ x\in X,\ \xi\in\Lambda\}=\cW
\end{equation}
is satisfied. To this end, suppose that we have $\eta_j\otimes u_j\in\Lambda\otimes \cW_j$, $j\in\bN_n$ such that
$$
((\eta_j\otimes u_j:j\in\bN_n),\phi(x,\xi))_{\cW}=\sum_{j=1}^n (\eta_j,\sqrt{B_j}\xi)_\Lambda(u_j,\phi_j(x))_{\cW_j}=0\mbox{ for all }x\in X\mbox{ and }\xi\in\Lambda.
$$
Note that $(u,\phi_j(\cdot))_{\cW_j}\in\cH_{\Psi_j}$ for each $j\in\bN_n$. We hence obtain by (\ref{exclusive}) that
$$
(\eta_j,\sqrt{B_j}\xi)_\Lambda(u_j,\phi_j(x))_{\cW_j}=0\mbox{ for all }j\in\bN_n,\ x\in X\mbox{ and }\xi\in\Lambda.
$$
By the denseness of $\phi_j(X)$ in $\cW_j$,
$$
(\eta_j,\sqrt{B_j}\xi)_\Lambda u_j=0\mbox{ for all }j\in\bN_n\mbox{ and }\xi\in\Lambda.
$$
We thus have for all $j\in\bN_n$ either $u_j=0$ or
$$
(\sqrt{B_j}\eta_j,\xi)_\Lambda=(\eta_j,\sqrt{B_j}\xi)_\Lambda=0\mbox{ for all }\xi\in\Lambda.
$$
In the latter case, we have that $\sqrt{B_j}\eta_j=0$. As $B_j$ is invertible, so is $\sqrt{B_j}$, following that $\eta_j=0$. In either case, we have that $\eta_j\otimes u_j=0$ for all $j\in\bN_n$. Equation (\ref{refinehilbertschmidt1eq1}) hence holds true. Similar facts hold for $\tilde{G}$.

By Proposition \ref{thirdcharacterization}, $\cH_K\preceq\cH_G$ is equivalent to $\cH_{\tilde{K}}\preceq\cH_{\tilde{G}}$, which by the above discussion and Lemma \ref{rkhsfeature} holds true if and only if for all $\xi_j\otimes u_j\in\Lambda\otimes\cW_j$, $j\in\bN_n$ there exist unique $\eta_j\otimes v_j\in\Lambda\otimes\cW_j$, $j\in\bN_m$ such that
\begin{equation}\label{refinehilbertschmidt1eq2}
\sum_{j=1}^n (\xi_j,\sqrt{B_j}\xi)_\Lambda(u_j,\phi_j(x))_{\cW_j}=\sum_{j=1}^m (\eta_j,\sqrt{C_j}\xi)_\Lambda(v_j,\phi_j(x))_{\cW_j} \mbox{ for all }\xi\in\Lambda\mbox{ and }x\in X
\end{equation}
and
\begin{equation}\label{refinehilbertschmidt1eq3}
\sum_{j=1}^n(\xi_j,\xi_j)_\Lambda(u_j,u_j)_{\cW_j}=\sum_{j=1}^m(\eta_j,\eta_j)_\Lambda(v_j,v_j)_{\cW_j}.
\end{equation}
Let $\xi_j\otimes u_j\in\Lambda\otimes\cW_j$, $j\in\bN_n$. If $B_j=C_j$ for $j\in\bN_n$ then we set $\eta_j:=\xi_j$ and $v_j:=u_j$ for $j\in\bN_n$, and $\eta_j=0$ and $v_j=0$ for $n+1\le j\le m$. Clearly, such a choice satisfies equations (\ref{refinehilbertschmidt1eq2}) and (\ref{refinehilbertschmidt1eq3}). Therefore, $\cH_{K}\preceq\cH_G$. Conversely, suppose that $\cH_K\preceq\cH_G$, that is, there exist $\eta_j\otimes v_j\in\Lambda\otimes\cW_j$, $j\in\bN_m$ that satisfy equations (\ref{refinehilbertschmidt1eq2}) and (\ref{refinehilbertschmidt1eq3}). Note that such $\eta_j\otimes v_j$ are unique by the denseness condition satisfied by the feature map for $\tilde{G}$. By (\ref{exclusive}), equation (\ref{refinehilbertschmidt1eq2}) implies that
$$
(\xi_j,\sqrt{B_j}\xi)_\Lambda(u_j,\phi_j(x))_{\cW_j}=(\eta_j,\sqrt{C_j}\xi)_\Lambda(v_j,\phi_j(x))_{\cW_j}\mbox{ for all }\xi\in\Lambda\mbox{ and }x\in X,\ \ j\in\bN_n
$$
and
$$
(\eta_j,\sqrt{C_j}\xi)_\Lambda(v_j,\phi_j(x))_{\cW_j}=0\mbox{ for all }\xi\in\Lambda\mbox{ and }x\in X,\ \ n+1\le j\le m.
$$
By the uniqueness of $\eta_j\otimes v_j\in\Lambda\otimes\cW_j$, $j\in\bN_m$, we must have that $\eta_j\otimes v_j=(\sqrt{C_j}^{\ -1}\sqrt{B_j}\xi_j)\otimes u_j$ for $j\in\bN_n$, and $\eta_j\otimes v_j=0$ for $n+1\le j\le m$. This together with (\ref{refinehilbertschmidt1eq3}) yields that
$$
\sum_{j=1}^n(\xi_j,\xi_j)_\Lambda(u_j,u_j)_{\cW_j}=\sum_{j=1}^n(\sqrt{B_j}C_j^{-1}\sqrt{B_j}\xi_j,\xi_j)_\Lambda(u_j,u_j)_{\cW_j}.
$$
By successively making $\xi_j\otimes u_j\ne0$ and $\xi_k\otimes u_k=0$ for $k\in\bN_n\setminus\{j\}$, for $j\in\bN_n$, we reach that
$$
(\xi_j,\xi_j)_\Lambda=(\sqrt{B_j}C_j^{-1}\sqrt{B_j}\xi_j,\xi_j)_\Lambda\mbox{ for all }\xi_j\in\Lambda\mbox{ and }j\in\bN_n.
$$
As $\sqrt{B_j}C_j^{-1}\sqrt{B_j}$ is hermitian, it equals the identity operator on $\Lambda$. It follows that $B_j=C_j$ for all $j\in\bN_n$. The proof is complete.
\end{proof}

As a corollary of Theorem \ref{refinehilbertschmidt1}, we obtain an orthogonal decomposition of $\cH_K$.

\begin{coro}\label{multidecomposition}
Let $K$ be defined by (\ref{hilbertschmidt1}), where $B_j$ are invertible and $\Psi_j$, $j\in\bN_n$ satisfy (\ref{exclusive}). Then
$$
\cH_K=\bigoplus_{j=1}^n \cH_{B_j\Psi_j}
$$
and
$$
\cH_{\sum_{j=1}^k B_j\Psi_j}\preceq \cH_{\sum_{j=1}^{k+1} B_j\Psi_j}\mbox{ for }k\in\bN_{n-1}.
$$
\end{coro}

A simplest case of (\ref{hilbertschmidt1}) occurs when $\cH_{\Psi_j}$ is of dimension $1$ for $j\in\bN_m$, which is covered below.

\begin{coro}\label{dimension1}
Let $B_j,C_k\in\cL_+(\Lambda)$ be invertible for $j\in\bN_n$ and $k\in\bN_m$, and $\psi_k:X\to\bC$, $k\in\bN_m$, be linearly independent. Set
$$
K(x,y):=\sum_{j=1}^n B_j\psi_j(x)\overline{\psi_j(y)},\ \ G(x,y):=\sum_{k=1}^m C_k\psi_k(x)\overline{\psi_k(y)},\ \ x,y\in X.
$$
Then $\cH_K\preceq\cH_G$ if and only if $B_j=C_j$ for all $j\in\bN_n$.
\end{coro}

More generally, we might consider $K,G$ defined by two distinct classes of linearly independent functions from $X$ to $\bC$. The result below can be proved using arguments similar to those for Theorem \ref{refinehilbertschmidt1}.

\begin{prop}\label{dimension1general}
Let $n\le m\in\bN_n$, $B_j,C_k\in\cL_+(\Lambda)$ be invertible for $j\in\bN_n$ and $k\in\bN_m$, and $\{\psi_j:j\in\bN_n\}$ and $\{\varphi_k:k\in\bN_m\}$ be two classes of linearly independent functions from $X$ to $\bC$. Set
$$
K(x,y):=\sum_{j=1}^n B_j\psi_j(x)\overline{\psi_j(y)},\ \ G(x,y):=\sum_{k=1}^m C_k\varphi_k(x)\overline{\varphi_k(y)},\ \ x,y\in X.
$$
Then $\cH_K\preceq\cH_G$ if and only if
\begin{description}
\item[(1)] $\psi_j\in\span\{\varphi_k:k\in\bN_m\}$ for all $j\in\bN_n$;

\item[(2)] the coefficients $\lambda_{jl}\in\bC$ in the linear span
$$
\psi_j=\sum_{l=1}^m \lambda_{jl}\varphi_l,\ \ j\in\bN_n
$$
satisfy
$$
\sum_{l=1}^m \lambda_{jl}\lambda_{kl}C_l^{-1}=\delta_{j,k}B_j^{-1}\mbox{ for all }j,k\in\bN_n.
$$
\end{description}
\end{prop}

We close this section with several concrete examples of finite Hilbert-Schmidt reproducing kernels of the form described in Corollary \ref{dimension1} and Proposition \ref{dimension1general}:
\begin{itemize}
\item polynomial kernels:
$$
K(x,y):=\sum_{j=1}^n x^{\alpha_j}\cdot y^{\alpha_j}B_j,\ \ x,y\in\bR^d
$$
where $\alpha_j$ are multi-indices and $B_j$ are invertible operators in $\cL_+(\Lambda)$,
or
$$
K(x,y):=\sum_{j=1}^n (x\cdot y)^{\beta_j}B_j,\ \ x,y\in\bR^d
$$
where $\beta_j$ are nonnegative integers.
\item exponential kernels:
$$
K(x,y):=\sum_{j=1}^n e^{i(x-y)\cdot t_j}B_j,\ \ x,y\in\bR^d
$$
where $t_j\in\bR^d$.
\end{itemize}

\section{Existence}
\setcounter{equation}{0}
This section is devoted to the existence of nontrivial refinement of operator-valued reproducing kernels. Most of the results to be presented here are straightforward extensions of those in the scalar-valued case \cite{XZ2}.

Let $X$ be the input space and $\Lambda$ be a Hilbert space. The reproducing kernels under consideration are $\cL(\Lambda)$-valued.

\begin{prop}\label{trivialexistence}
There does not exist a nontrivial refinement of an $\cL(\Lambda)$-valued reproducing kernel $K$ on $X$ if and only if $\cH_K=\Lambda^X$, the set of all the functions from $X$ to $\Lambda$. If the cardinality of $X$ is infinite then every $\cL(\Lambda)$-valued reproducing kernel on $X$ has a nontrivial refinement.
\end{prop}

Surprisingly, nontrivial results about the existence appear when $X$ is of finite cardinality.

\begin{prop}\label{finiteexistence}
Let $X$ consist of finitely many points $x_j$, $j\in\bN_n$ for some $n\in\bN_n$. A necessary condition for an $\cL(\Lambda)$-valued reproducing kernel on $X$ to have no nontrivial refinements is that
\begin{equation}\label{finitenecessary}
\sum_{j=1}^n\sum_{k=1}^n(K(x_j,x_k)\xi_j,\xi_k)_\Lambda>0\mbox{ for all }\xi_j\in\Lambda, j\in\bN_n\mbox{ with }\sum_{j=1}^n\|\xi_j\|_\Lambda>0.
\end{equation}
A sufficient condition for $K$ to have no nontrivial refinements is that
\begin{equation}\label{finitesufficient}
\sum_{j=1}^n\sum_{k=1}^n(K(x_j,x_k)\xi_j,\xi_k)_\Lambda\ge\lambda \sum_{j=1}^n\|\xi_j\|_\Lambda^2\mbox{ for all }\xi_j\in\Lambda, j\in\bN_n
\end{equation}
for some constant $\lambda>0$. Consequently, if $\Lambda$ is finite-dimensional then $K$ does not have a nontrivial refinement if and only if (\ref{finitenecessary}) holds true.
\end{prop}
\begin{proof}
Suppose that there exist $\xi_j\in\Lambda$, $j\in\bN_n$, at least one of which is nonzero, such that
$$
\sum_{j=1}^n\sum_{k=1}^n(K(x_j,x_k)\xi_j,\xi_k)_\Lambda=0.
$$
This implies that
$$
\sum_{j=1}^nK(x_j,\cdot)\xi_j=0.
$$
We get by (\ref{reproducing}) that for all $f\in\cH_K$
$$
\sum_{j=1}^n(f(x_j),\xi_j)_\Lambda=\biggl(f,\sum_{j=1}^nK(x_j,\cdot)\xi_j\biggr)_{\cH_K}=0.
$$
As a consequence, $\cH_K$ does not contain the function $f:X\to\Lambda$ taking values $f(x_j)=\xi_j$ for $j\in\bN_n$. By Proposition \ref{trivialexistence}, there exist nontrivial refinements for $K$ on $X$.

Suppose that (\ref{finitesufficient}) holds true for some positive constant $\lambda$. Assume that $\cH_K$ is a proper subset of $\Lambda^X$. Then there exists some nonzero vector $(\xi_k:k\in\bN_n)\in\Lambda^n$ orthogonal to $(f(x_k):k\in\bN_n)$ in $\Lambda^n$ for all $f\in\cH_K$. Letting $f=\sum_{j=1}^nK(x_j,\cdot)\xi_j$ yields that
$$
\sum_{j=1}^n\sum_{k=1}^n(K(x_j,x_k)\xi_j,\xi_k)_\Lambda=\sum_{k=1}^n(f(x_k),\xi_k)_\Lambda=0,
$$
contradicting (\ref{finitesufficient}).

We complete the proof by pointing out that when $\Lambda$ is finite-dimensional, (\ref{finitenecessary}) and (\ref{finitesufficient}) are equivalent.
\end{proof}

It is worthwhile to note that when $\Lambda$ is infinite-dimensional, condition (\ref{finitenecessary}) might not be sufficient for $K$ to not have a nontrivial refinement. We give a concrete example to illustrate this.

Let $X$ be a singleton $\{x\}$, $\Lambda:=\ell^2(\bN)$ consisting of square-summable sequences indexed by $\bN$, and $K(x_1,x_1)$ be the operator $T$ on $\ell^2(\bN)$ defined by
$$
Ta:=\left(\frac{a_j}j:j\in\bN\right),\ \ a\in\ell^2(\bN).
$$
Apparently, $T\in\cL_+(\ell^2(\bN))$ and condition (\ref{finitenecessary}) is satisfied. Let $f\in\cH_K$. Then there exist $a_n\in\ell^2(\bN)$, $n\in\bN$ such that $K(x,\cdot)a_n$ converges to $f$ in $\cH_K$. Being a Cauchy sequence in $\cH_K$, $\{K(x,\cdot)a_n:n\in\bN\}$ satisfies
$$
\lim_{n,m\to\infty}\|K(x,\cdot)a_n-K(x,\cdot)a_m\|^2_{\cH_K}=0.
$$
By (\ref{reproducing}),
$$
\begin{array}{rl}
\|K(x,\cdot)a_n-K(x,\cdot)a_m\|^2_{\cH_K}&=(K(x,\cdot)(a_n-a_m),K(x,\cdot)(a_n-a_m))_{\cH_K}\\
&=(K(x,x)(a_n-a_m),a_n-a_m)_{\ell^2(\bN)}
=(T(a_n-a_m),a_n-a_m)_{\ell^2(\bN)}\\
&=\|\sqrt{T}a_n-\sqrt{T}a_m\|_{\ell^2(\bN)}^2.
\end{array}
$$
Combining the above two equations yields $\sqrt{T}a_n$ converges to some $b\in\ell^2(\bN_n)$. We now have for each $c\in\ell^2(\bN)$ that
$$
\begin{array}{rl}
(f(x),c)_{\ell^2(\bN)}&\displaystyle{=(f,K(x,\cdot)c)_{\cH_K}=\lim_{n\to\infty}(K(x,\cdot)a_n,K(x,\cdot)c)_{\cH_K}}\\
&\displaystyle{=\lim_{n\to\infty}(K(x,x)a_n,c)_{\ell^2(\bN)}=\lim_{n\to\infty}(Ta_n,c)_{\ell^2(\bN)}}\\
&\displaystyle{=\lim_{n\to\infty}(\sqrt{T}a_n,\sqrt{T}c)_{\ell^2(\bN)}=(b,\sqrt{T}c)_{\ell^2(\bN)}}\\
&=(\sqrt{T}b,c)_{\ell^2(\bN)},
\end{array}
$$
which implies that $f(x)=\sqrt{T}b$. Since this is true for an arbitrary function $f\in\cH_K$, the function $g:X\to\Lambda$ defined by
$$
g(x):=\left(\frac1j:j\in\bN\right)
$$
is not in $\cH_K$. Thus, $K$ has a nontrivial refinement on $X$.

In the process of refining an operator-valued reproducing kernel, it is usually desirable to preserve favorable properties of the original kernel. We shall show that this is feasible as far as continuity and universality of operator-valued reproducing kernels are concerned. Let $X$ be a metric space and $K$ an $\cL(\Lambda)$-valued reproducing kernel that is continuous from $X\times X$ to $\cL(\Lambda)$ when the latter is equipped with the operator norm. Then one sees that $\cH_K$ consists of continuous functions from $X$ to $\Lambda$. For each compact subset $\cZ\subseteq X$, denote by $\cC(\cZ,\Lambda)$ the Banach space of all the continuous functions from $\cZ$ to $\Lambda$ with the norm
$$
\|f\|_{\cC(\cZ,\Lambda)}:=\max_{x\in \cZ}\|f(x)\|_{\Lambda},\ \ f\in\cC(\cZ,\Lambda).
$$
Following \cite{MXZ} and \cite{CMPY}, we call $K$ a {\it universal kernel} on $X$ if for all compact sets $\cZ\subseteq X$ and all continuous functions $f:X\to\Lambda$ there exist
$$
f_n\in\span\{K(x,\cdot)\xi:x\in \cZ,\ \xi\in\Lambda\},\ \ n\in\bN,
$$
such that
$$
\lim_{n\to\infty}\|f_n-f\|_{\cC(\cZ,\Lambda)}=0.
$$
In other words, $K$ is universal if for all compact subsets $\cZ\subseteq X$, the closure of $\span\{K(x,\cdot)\xi:x\in\cZ,\ \xi\in\Lambda\}$ in $\cC(\cZ,\Lambda)$ equals the whose space $\cC(\cZ,\Lambda)$.

For the preservation of continuity, we have the following affirmative result, whose proof is similar to the scalar-valued case \cite{XZ2}.

\begin{prop}\label{continuity}
Let $X$ be a metric space with infinite cardinality. Then every continuous $\cL(\Lambda)$-valued reproducing kernel on $X$ has a nontrivial continuous refinement.
\end{prop}

The following lemma about universality has been proved in \cite{CMPY}, and in \cite{MXZ} in the scalar-valued case. We provide a simplified proof here.

\begin{lemma}\label{universal}
Let $K$ be a continuous $\cL(\Lambda)$-valued reproducing kernel on $X$ with the feature map representation (\ref{featureK}), where $\Phi:X\to\cL(\Lambda,\cW)$ is continuous. Then for each compact subset $\cZ\subseteq X$,
$$
\overline{\span}\{K(x,\cdot)\xi:x\in \cZ,\ \xi\in\Lambda\}=\overline{\{\Phi(\cdot)^*u:u\in\cW\}},
$$
where the closures are relative to the norm in $\cC(\cZ,\Lambda)$.
\end{lemma}
\begin{proof}
All the closures to appear in the proof are relative to the norm in $\cC(\cZ,\Lambda)$. Let $K_\cZ$ be the restriction of $K$ on $\cZ$. Then the restriction of $\Phi$ on $\cZ$ remains a feature map for $K_\cZ$. By Lemma \ref{rkhsfeature},
\begin{equation}\label{universaleq1}
\cH_{K_\cZ}=\{\Phi(\cdot)^*u:u\in\cW\}.
\end{equation}
It hence suffices to show that
$$
\overline{\span}\{K(x,\cdot)\xi:x\in \cZ,\ \xi\in\Lambda\}=\overline{\span}\{K_\cZ(x,\cdot)\xi:x\in \cZ,\ \xi\in\Lambda\}=\overline{\cH_{K_\cZ}}.
$$
As $\span\{K_\cZ(x,\cdot)\xi:x\in \cZ,\ \xi\in\Lambda\}\subseteq\cH_{K_\cZ}$,
\begin{equation}\label{universaleq2}
\overline{\span}\{K_\cZ(x,\cdot)\xi:x\in \cZ,\ \xi\in\Lambda\}\subseteq\overline{\cH_{K_\cZ}}.
\end{equation}
On the other hand, for each $f\in\cH_{K_\cZ}$ there exist $f_n\in\span\{K_\cZ(x,\cdot)\xi:x\in \cZ,\ \xi\in\Lambda\}$, $n\in\bN$ that converges to $f$ in the norm of $\cH_{K_\cZ}$. It follows that $f_n$ converges to $f$ in the norm of $\cC(\cZ,\Lambda)$. Therefore, $f\in \overline{\span}\{K_\cZ(x,\cdot)\xi:x\in \cZ,\ \xi\in\Lambda\}$, implying that
\begin{equation}\label{universaleq3}
\overline{\cH_{K_\cZ}}\subseteq\overline{\span}\{K_\cZ(x,\cdot)\xi:x\in \cZ,\ \xi\in\Lambda\}.
\end{equation}
Combining equations (\ref{universaleq1}), (\ref{universaleq2}), and (\ref{universaleq3}) proves the result.
\end{proof}

The following positive result about universality can be proved by Lemma \ref{universal} and arguments similar to those used in Proposition 14 of \cite{XZ2}.

\begin{prop}
Let $X$ be a metric space and $K$ a continuous $\cL(\Lambda)$-valued reproducing kernel on $X$. Then every continuous refinement of $K$ on $X$ remains universal.
\end{prop}

\section{Numerical Experiments}
\setcounter{equation}{0}

We present in this final section two numerical experiments on the application of refinement of operator-valued reproducing kernels to multi-task learning. Suppose that $f_0$ is a function from the input space $X$ to the output space $\Lambda$ that we desire to learn from its finite sample data $\{(x_j,\xi_j):j\in\bN_m\}\subseteq X\times \Lambda$. Here $m$ is the number of sampling points and
$$
\xi_j=f_0(x_j)+\delta_j,\ \ j\in\bN_m
$$
where $\delta_j\in\Lambda$ is the noise dominated by some unknown probability measure. To deal with the noise and have an acceptable generalization error, we use the following regularization network
\begin{equation}\label{regularizationworks1}
\min_{f\in\cH_K}\frac1m\sum_{j=1}^m\|f(x_j)-\xi_j\|_\Lambda^2+\sigma \|f\|_{\cH_K}^2,
\end{equation}
where $K$ is a chosen $\Lambda$-valued reproducing kernel on $X$. Our experiments will be designed so that underfitting and overfitting both have the chance to occur. To echo with the motivations in Section 2, when underfitting happens in the first experiment, we shall find a refinement $G$ of $K$ aiming at improving the performance of the minimizer of (\ref{regularizationworks1}) in prediction. On the other hand, when overfitting appears in the second experiment, we shall then find a $\Lambda$-valued reproducing kernel $L$ on $X$ such that $\cH_L\preceq\cH_K$ with the same purpose.

Before moving on to the experiments, we make a remark on how (\ref{regularizationworks1}) can be solved. The issue has been understood in the work \cite{MP2005}. We say that $K$ is {\it strictly positive-definite} if for all finite $y_j\in X$, $j\in\bN_p$, and for all $\eta_j\in\Lambda$, $j\in\bN_p$ all of which are not zero
$$
\sum_{j=1}^p\sum_{k=1}^p(K(y_j,y_k)\eta_j,\eta_k)_\Lambda>0.
$$
If $K$ is strictly positive-definite then the minimizer $f_K$ of (\ref{regularizationworks1}) has the form
\begin{equation}\label{representer}
f_K=\sum_{j=1}^mK(x_j,\cdot)\eta_j
\end{equation}
where $\eta_j$'s satisfy
\begin{equation}\label{linearsystem}
\sum_{k=1}^m K(x_k,x_j)\eta_k +m\sigma \eta_j=\xi_j,\ \ j\in\bN_m.
\end{equation}

\subsection{Experiment one: underfitting}
The vector-valued function to be learned from finite examples is from the input space $X=[-1,1]$ to output space $\Lambda=\bR^n$, where $n\in\bN$. Specifically, it has the form
\begin{equation}\label{functionf01}
f_0(x):=\left[a_k|x-b_k|+c_k e^{-d_k x}:k\in\bN_n\right],\ \ x\in[-1,1],
\end{equation}
where $a,b,c,d$ are constant vectors to be randomly generated. The $\cL_+(\bR^n)$-valued reproducing kernel that we shall use in the regularization network (\ref{regularizationworks1}) is a Gaussian kernel
$$
K(x,y):=S\exp\left(-\frac{(x-y)^2}2\right),\ \ x,y\in [-1,1],
$$
where $S\in \cL_+(\bR^n)$ is strictly positive-definite. It can be identified by Lemma \ref{rkhsfeature} that functions in $\cH_K$ are of the form $\sqrt{S} v$, where $v$ is an $\bR^n$-valued function on $[-1,1]$ such that for each $k\in\bN_n$, its $k$-th component $v_k$ is the restriction on $[-1,1]$ of a square Lebesgue integrable function $u_k$ on $\bR$ such that
$$
\int_{\bR}\left|\hat{u_k}(t)\right|^2\exp\left(\frac{t^2}2\right)dt<+\infty.
$$
Here $\hat{u_k}$ denotes the Fourier transform of $u_k$ given as
$$
\hat{u_k}(t):=\frac1{\sqrt{2\pi}}\int_\bR e^{-ixt}u_k(x)dx,\ \ t\in\bR.
$$
Therefore, such a function $u_k$ can be extended to an analytic function of finite order on the complex plane. In particular, it implies that each component $v_k$ of $v$ is real-analytic on $[-1,1]$. As a result, components of functions in $\cH_K$ are real-analytic. The function $f_0$ to be approximated is defined by (\ref{functionf01}). We see that while the exponential component $e^{-d_k x}$ is real-analytic, the first component $|x-b_k|$ is not even continuously differentiable. Underfitting is hence expected. If this is indeed observed then a remedy is to use the refinement of $K$ given by
$$
G(x,y):=S\exp\left(-\frac{(x-y)^2}2\right)+T(1+xy)^3,\ \ x,y\in[-1,1],
$$
where $T\in \cL_+(\bR^n)$ is also strictly positive-definite. It can be verified that $\cH_K\cap\cH_{G-K}=\{0\}$. By Proposition \ref{firstcharacterization}, $G$ is a nontrivial refinement of $K$. Furthermore, as low order polynomials are introduced, the ability for functions in $\cH_G$ to approximate the function $|x-b_k|$ is expected to be superior to those in $\cH_K$. We perform extensive numerical simulations to confirm these conjectures.

The dimension $n$ will be chosen from $\{2,4,8,16\}$. The number $m$ of sampling points will be set to be $30$. The sampling points $x_j$, $j\in\bN_m$ will be randomly sampled from $[-1,1]$ by the uniform distribution and the outputs $\xi_j$ are generated by
\begin{equation}\label{noisedoutput}
\xi_j=f_0(x_j)+\delta_j,\ \ j\in\bN_m,
\end{equation}
where $\delta_j$ are vectors whose components will be randomly generated by the uniform distribution on $[-\delta,\delta]$ with $\delta$ being the noise level selected from $\{0.1,0.3,0.5\}$. For each dimension $n\in\{2,4,8,16\}$ and noise level $\delta\in\{0.1,0.3,0.5\}$, we run 50 simulations. In each of the simulations, we do the followings:
\begin{enumerate}
\item the components of the coefficient vectors $a,b,c,d$ in the function $f_0$ given by (\ref{functionf01}) are randomly generated by the uniform distribution on $[1,3]$, $[-1,1]$, $[-2,2]$, and $[0,3]$, respectively;

\item the sampling points are randomly sampled from $[-1,1]$ by the uniform distribution and the outputs $\xi_j$ are then generated by (\ref{noisedoutput});

\item the matrices $S$ and $T$ are given by $S=A'A$ and $T=B'B$ where $A,B$ are $n\times n$ real matrices whose components are randomly sampled from $[1,3]$ by the uniform distribution;

\item we then solve the minimizer $f_K$ of (\ref{regularizationworks1}) by (\ref{representer}) and (\ref{linearsystem});

\item for the refinement kernel $G$, we also obtain $f_G$ as the minimizer of
\begin{equation}\label{regularizationworks2}
\min_{f\in\cH_G}\frac1m\sum_{j=1}^m\|f(x_j)-\xi_j\|_\Lambda^2+\sigma \|f\|_{\cH_G}^2,
\end{equation}

\item the regularization parameters in (\ref{regularizationworks1}) and (\ref{regularizationworks2}) are optimally chosen so that the relative square approximation errors
    \begin{equation}\label{relativeerror}
    \cE_K:=\frac{\int_{-1}^1\|f_K(t)-f_0(t)\|^2dt}{\int_{-1}^1\|f_0(t)\|^2dt},\ \ \cE_G:=\frac{\int_{-1}^1\|f_G(t)-f_0(t)\|^2dt}{\int_{-1}^1\|f_0(t)\|^2dt}.
    \end{equation}
    are minimized, respectively.
\end{enumerate}

We call $(\cE_K,\cE_G)$ obtained in each simulation an instance of approximation errors. Hence, we have 50 instances for each pair of $(n,\delta)$. They are said to form a group. There are 12 groups of instances of approximation errors. For each $(n,\delta)$, we shall calculate the mean and standard deviation of the difference $\cE_K-\cE_G$ in the corresponding group as a measurement of the difference in the performance of learning schemes (\ref{regularizationworks1}) and (\ref{regularizationworks2}). Before that, outliers of instances should be excluded. Although we do not know the distributions of $\cE_K$ and $\cE_G$, we shall use the three-sigma rule in statistics. In other words, we regard an instance $(\cE_K,\cE_G)$ as an outlier if the deviation of $\cE_K$ or $\cE_G$ to their respective mean in the group exceeds three times their respective standard deviation. There are 32 outliers among the entire 600 instances, which are listed below in Table 7.1.\newpage

\noindent{\bf Table 7.1} {\bf Outliers of instances of approximation errors $(\cE_K,\cE_G)$}. {\it An instance $(\cE_K,\cE_L)$ is considered to be an outlier if the deviation of one of its components to the respective mean in the group is more than three times the standard deviation of the group. Outliers are listed in an independent table because they should be excluded from the calculation of the mean and standard deviation of the approximation errors. Another reason is that adding them will make the plot of the approximation errors highly disproportional.}
$$
\begin{array}{|c|c|c|c|c|}\hline
&n=2&n=4&n=8&n=16\\\hline
\multirow{4}{*}{$\delta=0.1$}&(0.1024,0.0084)&(0.0215,0.0182)&(0.0230,0.0070)&(0.0712,0.0015)\\
                             &(0.0091,0.0081)&(0.4095,0.0034)&(0.0513,0.0091)&(0.0364,0.0124)\\
                             &(0.4128,0.0006)&               &(0.1554,0.0011)&               \\
                             &(0.6783,0.0025)&               &(0.1464,0.0026)&               \\\hline
\multirow{3}{*}{$\delta=0.3$}&(0.0286,0.0228)&(0.0663,0.0321)&(0.0407,0.0194)&(0.1592,0.0018)\\
                             &(0.4811,0.0020)&(0.1892,0.0041)&(0.1809,0.0023)&(0.0309,0.0127)\\
                             &               &(0.1674,0.0095)&               &(0.0229,0.0099)\\\hline
\multirow{3}{*}{$\delta=0.5$}&(0.2053,0.0020)&(0.0377,0.0376)&(0.2445,0.0028)&(0.1612,0.0043)\\
                             &(0.1267,0.0034)&(0.3547,0.0033)&(0.2762,0.0020)&(0.0541,0.0081)\\
                             &(0.0669,0.0465)&               &(0.0119,0.0264)&               \\\hline
  \end{array}
$$

We make a few observations from Table 7.1. Firstly, $\cE_G$ is smaller than $\cE_K$ except for only one instance. For a large portion of the outliers, the approximation error $\cE_K$ is considerably large (larger than 10\%), a sign of underfitting of the kernel $K$. Those instances are of the greatest interest to us as we desire to see if the refinement kernel $G$ can make a remedy when overfitting does happen. We see from Table 7.1 that for all of those outliers, the refinement kernel $G$ always brings down the relative approximation error to be less than $1\%$. The improvement brought by $G$ for other instances is also significant. The observations indicate that (\ref{regularizationworks2}) performs significantly better in learning the function (\ref{functionf01}) from finite examples than (\ref{regularizationworks1}). For further comparison, we compute the mean and standard deviation of the difference $\cE_K-\cE_G$ of the approximation errors after excluding the above outliers. The results are tabulated below. Note that a positive value of the mean implies that (\ref{regularizationworks2}) performs better than (\ref{regularizationworks1}). It is worthwhile to point out that among all the rest 568 instances excluding the outliers, there are only 33 where $\cE_G$ is larger than $\cE_K$. The largest value of $\cE_G-\cE_K$ is $0.0020$. Therefore, we conclude that for all the $(n,\delta)$, (\ref{regularizationworks2}) is superior to (\ref{regularizationworks1}), and the larger the standard deviation in Table 7.2 is, the greater improvement the refinement kernel $G$ brings.\newline

\noindent {\bf Table 7.2 The mean and standard deviation (in parentheses) of $\cE_K-\cE_G$.} {\it The outliers of instances listed in Table 7.1 are not counted toward these calculations. If they were added, the improvement brought by the refinement kernel $G$ would have been more dramatic.}
$$
\begin{array}{|c|c|c|c|c|}\hline
          & n=2         &n=4      &n=8       &n=16\\\hline
\multirow{2}{*}{$\delta=0.1$}&0.0098  &  0.0139 &0.0160  & 0.0108 \\
                             &(0.0182)& (0.0335)&(0.0241)&(0.0135)\\\hline
\multirow{2}{*}{$\delta=0.3$}&0.0076  &  0.0141 &0.0143  & 0.0188\\
                             &(0.0144)& (0.0245)&(0.0208)&(0.0259)\\\hline
\multirow{2}{*}{$\delta=0.5$}& 0.0054 &  0.0127 &0.0103  & 0.0091\\
                             &(0.0121)& (0.0307)&(0.0186)&(0.0102)\\\hline
\end{array}
$$\newline

We shall also plot the 12 groups of approximation errors $\cE_K,\cE_G$ for a visual comparison. To this end, we take out the instances for which $\cE_K$ is too large to have an appropriate range in the vertical axes in the figures. Therefore, Figures 7.1 and 7.2 are not full embodiment of the improvement of (\ref{regularizationworks2}) over (\ref{regularizationworks1}). Nevertheless, one sees that the improvement brought by the refinement kernel $G$ in these relatively well-behaved instances is still dramatic.\newline

\noindent{\bf Figure 7.1 Relative approximation errors $\cE_K,\cE_G$ for $n=2,4$ and $\delta=0.1,0.3,0.5$.} {\it The outliers listed in Table 7.1 are not plotted here as they would make the figure highly disproportional.}
\begin{center}
\scalebox{0.5}[1]{\includegraphics*{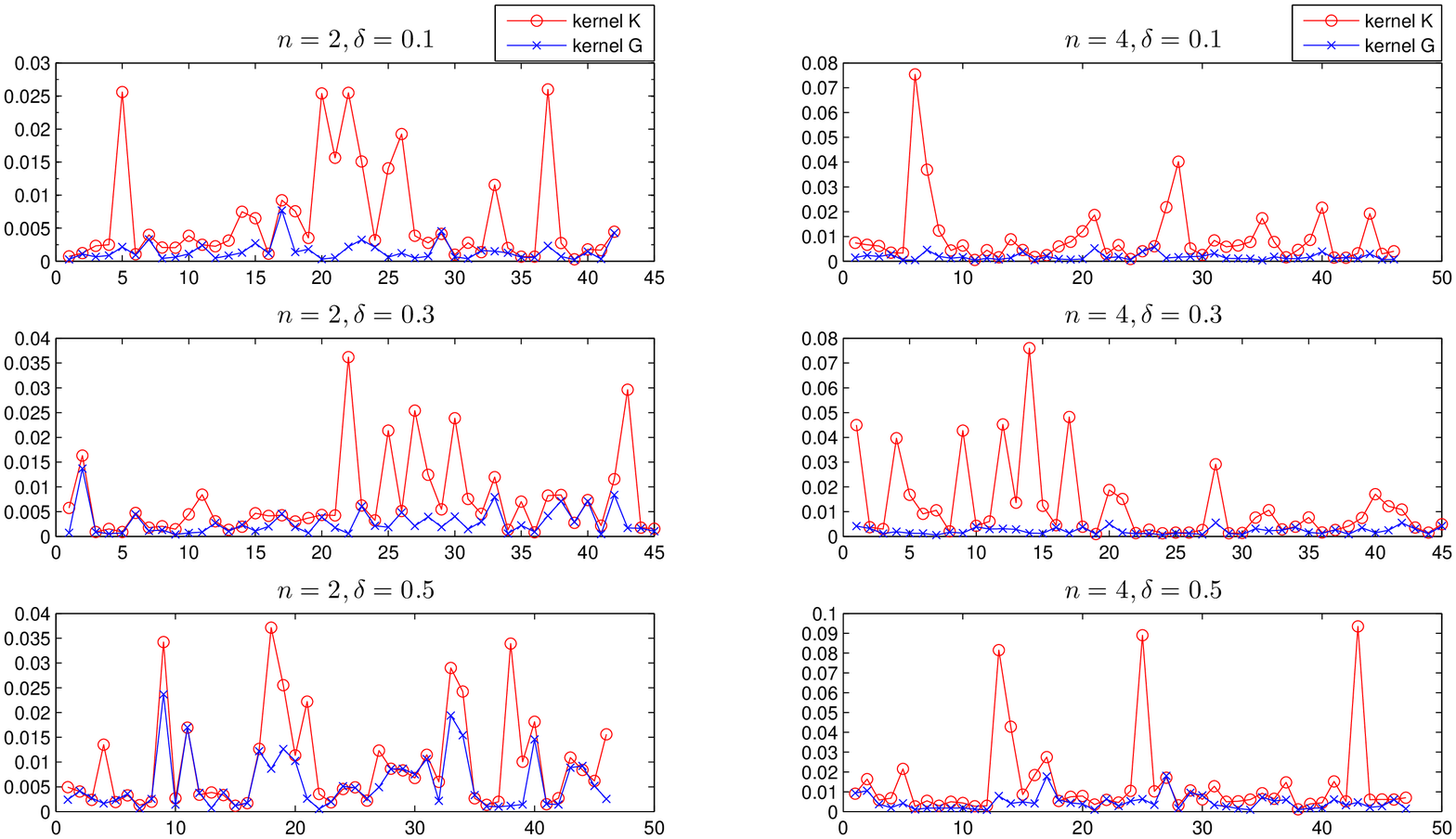}}
\end{center}
\newpage
\noindent{\bf Figure 7.2 Relative approximation errors $\cE_K,\cE_G$ for $n=8,16$ and $\delta=0.1,0.3,0.5$.} {\it The outliers listed in Table 7.1 are not plotted in the figure here.}
\begin{center}
\scalebox{0.5}[1.1]{\includegraphics*{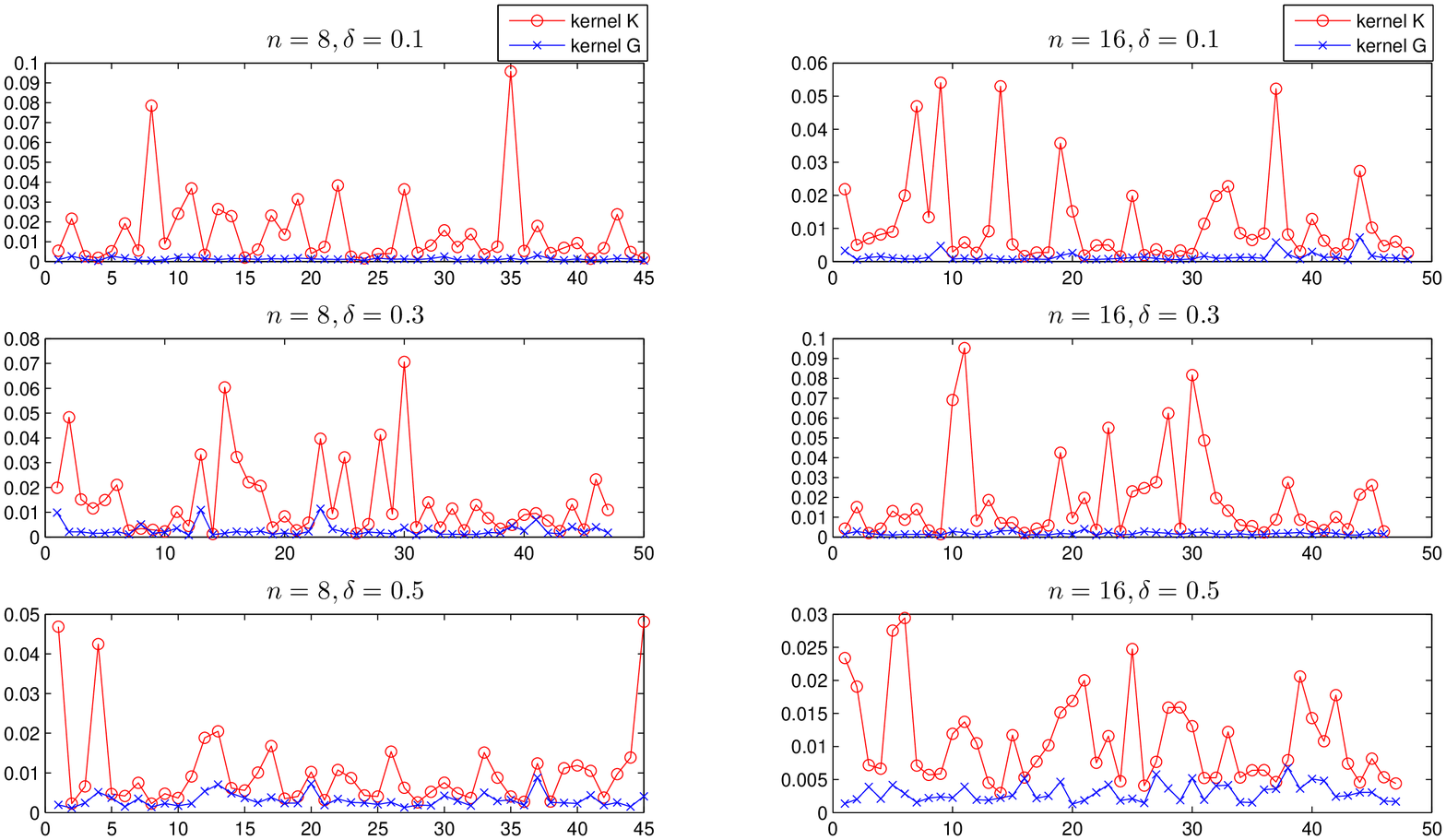}}
\end{center}

\subsection{Experiment 2: overfitting}

The target function we consider in the second experiment is
\begin{equation}\label{functionf02}
f_0(x)=\left[\frac{a_k}{1+25(x-b_k)^2}+c_ke^{-d_k x}:k\in\bN_n\right],\ \ x\in[-1,1],
\end{equation}
where the components of the vectors $a,b,c,d\in\bR^n$ will be randomly sampled by the uniform distribution from $[1,4]$, $[0,\frac12]$, $[-2,2]$, and $[0,2]$ respectively in the numerical simulations. The dimension $n$ will be chosen from $\{2,4,8,16\}$. We fix $m:=20$ and shall sample the inputs $x_j$, $j\in\bN_m$ randomly by the uniform distribution from $[-1,1]$. Similarly, the outputs $\xi_j\in\bR^n$, $j\in\bN_m$ will be generated by (\ref{noisedoutput}) where the noise level is to be selected from $\{0.1,0.3,0.5\}$.

In the first step, we substitute the sample data $\{(x_j,\xi_j):j\in\bN_m\}$ into the regularization network (\ref{regularizationworks1}) with the following kernel
\begin{equation}\label{kernelKoverfitting}
K(x,y):=S\exp\left(-\frac{(x-y)^2}2\right)+T(1+xy)^{18},\ \ x,y\in [-1,1],
\end{equation}
where $S=A'A$ and $T=B'B$ with $A,B$ being $n\times n$ real-matrices whose components will be randomly sampled by the uniform distribution from $[1,2]$. The target function (\ref{functionf02}) contains translations of the Runge function
$$
\frac1{1+25x^2},\ \ x\in[-1,1].
$$
It is well-known that approximating the Runge function by high order polynomial interpolations leads to overfitting. One sees by (\ref{linearsystem}) that the regulation network (\ref{regularizationworks1}) might be regarded as a regularized interpolation. Note also that the order of the polynomial kernel in (\ref{kernelKoverfitting}) is $18$, which is close to the number $m=20$ of sampling points. Overfitting is hence expected. When this occurs, we propose to reduce the order of the polynomial kernel by considering
$$
L(x,y):=S\exp\left(-\frac{(x-y)^2}2\right)+T\sum_{k=0}^{10}{{18}\choose{k}}(xy)^k,\ \ x,y\in[-1,1].
$$
By Corollary \ref{multidecomposition}, $\cH_L\preceq\cH_K$, namely, $K$ is a refinement of $L$. We shall demonstrate by numerical simulations that
\begin{equation}\label{regularizationworks3}
\min_{f\in\cH_L}\frac1m\sum_{j=1}^m \|f(x_j)-\xi_j\|^2+\sigma \|f\|_{\cH_L}^2
\end{equation}
outperforms (\ref{regularizationworks1}) with the kernel (\ref{kernelKoverfitting}). To this end, we shall conduct numerical experiments similar to those in the last subsection. Let $f_K$ and $f_L$ be the minimizer of (\ref{regularizationworks1}) and (\ref{regularizationworks3}), respectively. We shall measure the performance by the relative square approximation errors $\cE_K$ and $\cE_L$, which are defined in the same way as (\ref{relativeerror}). For each pair of $(n,\delta)$, where $n\in\{2,4,8,16\}$ and $\delta\in\{0.1,0.3,0.5\}$, we run $20$ numerical simulations where the regularization parameters $\sigma$ are to be chosen so that $\cE_K$ and $\cE_L$ are minimized, respectively. As in the first experiment, we shall calculate the mean and standard deviation of $\cE_K$ and $\cE_L$ in each group after taking out some outliers. We shall also plot the relative errors for comparison. The results are shown below in the form of tables and figures.

\noindent{\bf Table 7.3 Outliers of instances of relative approximation errors $(\cE_K,\cE_L)$. }
$$
\begin{array}{|c|c|c|l|}\hline
&\delta=0.1&\delta=0.3&\delta=0.5\\\hline
\multirow{2}{*}{n=2}&(0.9000,  0.7843)&(2.9906,1.3509)&(1.8065, 0.8044),(1.1332, 0.3213)\\
                    &                 &               &(19.6416, 7.6578)\\\hline
  \multirow{4}{*}{n=4}&(8.2450,5.8717)&(1.1760,0.1354)&(4.6316,7.0497),(2.0850,1.3204)\\
                      &(1.6654,2.0466)&(  0.4591, 0.7845)&(2.4657,1.1386)\\
                      &(18.9615,12.0513)&               &(5.7967,0.6122)\\
                     &(0.9536,  1.0998)&                 &(5.1196,2.6692)\\\hline
\multirow{7}{*}{n=8}&(0.9102,1.3862)&(1.3517, 1.8339)&(0.6369, 0.3698),(0.6945,0.2878)\\
                    &(1.2233,0.9489) &(0.8450,0.2605)&(2.2371, 2.4008)\\
                    &(0.6711,0.2249)&(0.3571, 0.7221)&(1.0738,0.4172)\\
                    &&(2.2403, 2.0108)&(1.0561,0.3067)\\
                    &              &(5.6153,5.0954)&(0.6791,1.0980)\\
                    &               &(2.0763,1.3718)&(3.6689,3.9566)\\
                    &               &(2.2567,1.4024)&(1.1238,0.2467)\\\hline
\multirow{5}{*}{n=16}&(4.4905,    5.8886)&(26.0758,7.6125)&(73.0854,42.6904),(1.6070,    1.4224)\\
                    &(7.9187,    4.3445)&(1.2255,    0.3181)&(3.2674, 2.2622),(2.1632,    1.7059)\\
                    &(2.1619, 0.5061)&(0.5140,    0.1817)&(2.8067,    0.5791),(9.0120,   3.5443)\\
                    &(17.5145,   13.7894)&(2.4289,    1.9022)&(0.6064,    0.3365),(4.0484 ,   0.4220)\\
                    &                  &                   &(1.0064,    0.8287)\\\hline
  \end{array}
$$\newline

We have more outliers compared to the first experiment. Using fewer sampling points and approximating the Runge function by polynomials both contributes to this. We observe that for the majority of these outliers, $\cE_L$ is significantly smaller than $\cE_K$, showing improvement of learning scheme (\ref{regularizationworks3}) over (\ref{regularizationworks1}). For further comparison, we shall compute the mean and variances of $\cE_K-\cE_L$ and plot the relative approximation errors $\cE_K$ and $\cE_L$ for the rest of instances.\newline

\noindent {\bf Table 7.4 The mean and standard deviation (in parentheses) of $\cE_K-\cE_L$.} {\it The outliers of instances listed in Table 7.3 are not counted toward these calculations. If they were added, the improvement brought by the refinement kernel $G$ would have been more dramatic.}

$$
\begin{array}{|c|c|c|c|c|}\hline
          & n=2         &n=4      &n=8       &n=16\\\hline
\multirow{2}{*}{$\delta=0.1$}&0.0289&    0.0511&    0.0173&    0.0157\\
    &(0.0846)&(0.0587)&(0.0779)&(0.0146)\\\hline
\multirow{2}{*}{$\delta=0.3$}&0.0404& 0.0661& 0.0671& 0.0657\\
    &(0.0922)&(0.0705)&(0.0929)&(0.0918)\\\hline
\multirow{2}{*}{$\delta=0.5$}& 0.0629  &  0.0130&    0.0484&    0.0625\\
   &(0.1098)&(0.0233)&(0.0758)&(0.0821)\\\hline
\end{array}
$$\newline

A positive value of the mean in Table 7.4 implies that (\ref{regularizationworks3}) performs better than (\ref{regularizationworks1}). It is observed that kernel $L$ brings improvement for all the choices of $n\in\{2,4,8,16\}$ and $\delta\in\{0.1,0.3,0.5\}$. We also remark that among all the 188 instances counted in Table 7.4, there are only 32 for which $\cE_L>\cE_K$. The mean and standard deviation of $\cE_L-\cE_K$ for these 32 instances are $0.0264$ and $0.0306$. We conclude that compared to (\ref{regularizationworks1}), (\ref{regularizationworks3}) improves the performance considerably in learning the function (\ref{functionf02}).\newline

\noindent{\bf Figure 7.3 Relative approximation errors $\cE_K,\cE_L$ for $n=2,4$ and $\delta=0.1,0.3,0.5$.} {\it The outliers listed in Table 7.3 are not plotted here as they will make the figure highly disproportional.}
\begin{center}
\scalebox{0.5}[1]{\includegraphics*{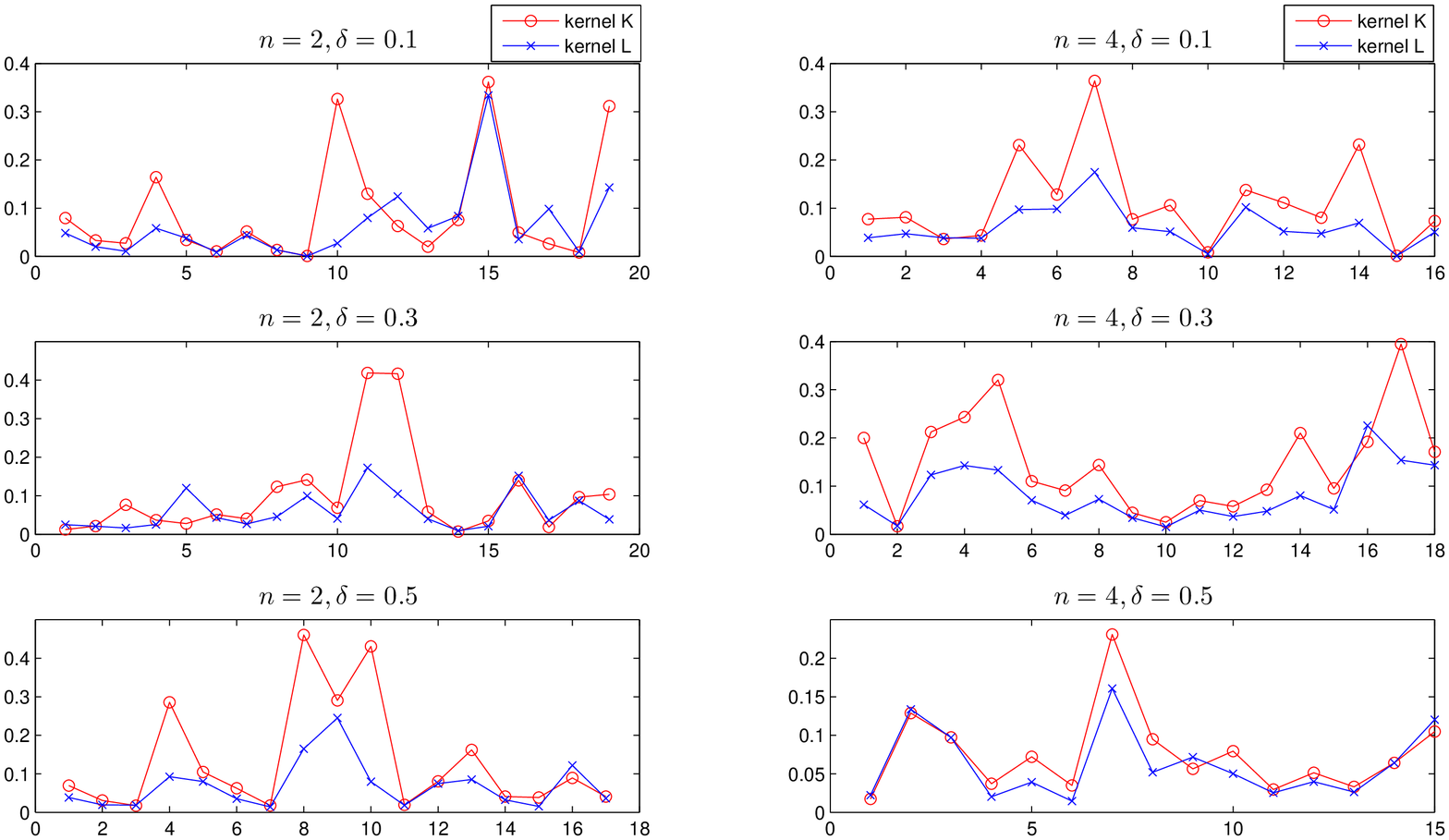}}
\end{center}

\newpage
\noindent{\bf Figure 7.4 Relative approximation errors $\cE_K,\cE_L$ for $n=8,16$ and $\delta=0.1,0.3,0.5$.} {\it The outliers listed in Table 7.3 are not plotted here.}
\begin{center}
\scalebox{0.5}[0.8]{\includegraphics*{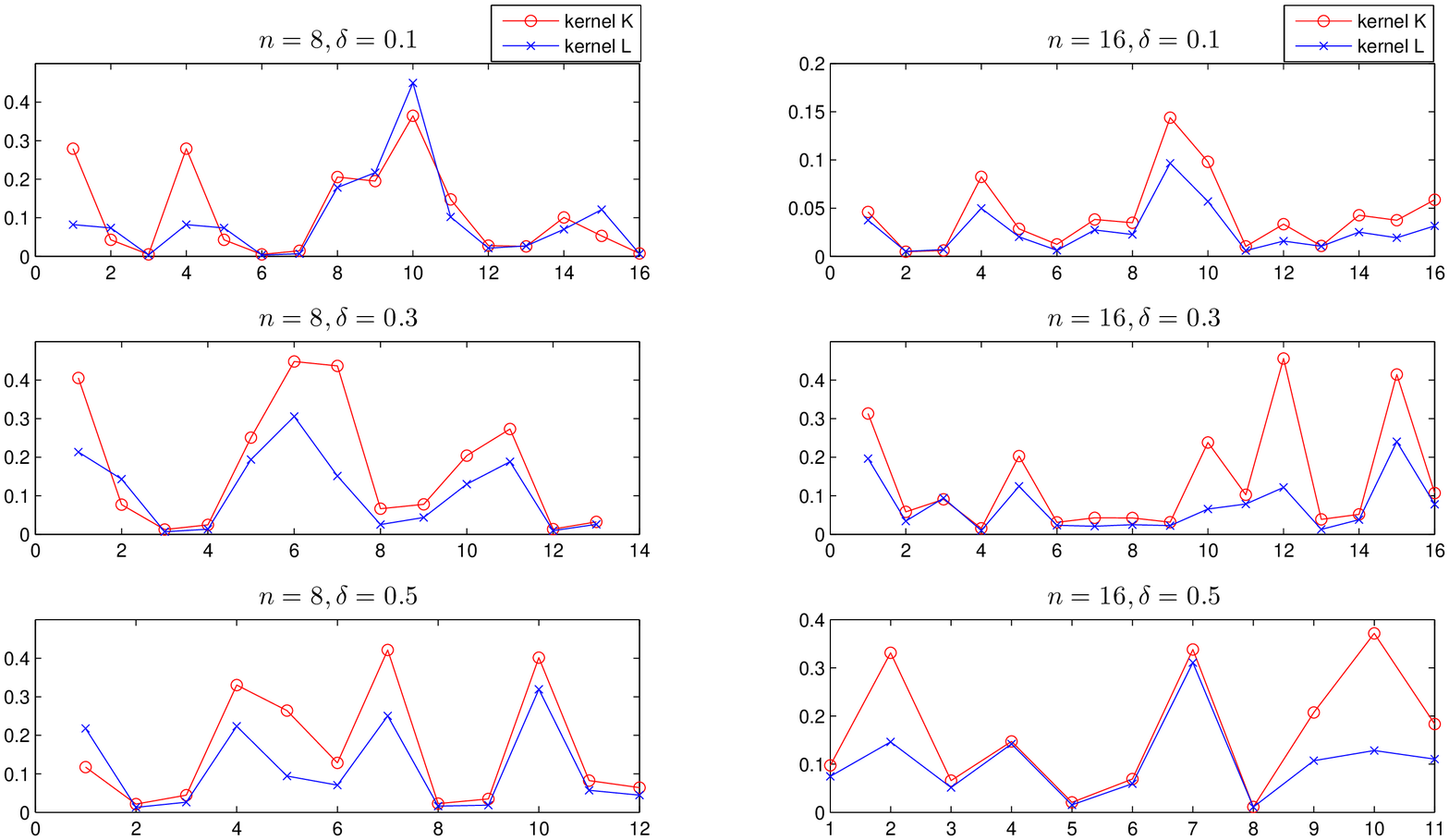}}
\end{center}

\section{Conclusion}

The refinement relationship between two operator-valued reproducing
kernels provides a promising way of updating kernels for multi-task
machine learning when overfitting or underfitting occurs. We
establish several general characterizations of the refinement
relationship. Particular attention has been paid to the case when
the kernels under investigation have a vector-valued integral
representation, the most general form of operator-valued reproducing
kernels. By the characterizations, we present concrete examples of
refining the translation invariant operator-valued reproducing
kernels, Hessian of the scalar-valued Gaussian kernel, and finite
Hilbert-Schmidt operator-valued reproducing kernels. Two numerical
experiments confirm the potential usefulness of the proposed
refinement method in updating kernels for multi-task learning. We
plan to investigate the effect of the method by real application
data in another occasion.

\end{document}